\documentclass[lettersize,journal]{IEEEtran}
\usepackage{amsmath,amsfonts,amsthm}
\usepackage{algorithmic}
\usepackage{algorithm}
\usepackage{array}
\usepackage[caption=false,font=normalsize,labelfont=sf,textfont=sf]{subfig}
\usepackage{textcomp}
\usepackage{stfloats}
\usepackage{url}
\usepackage{verbatim}
\usepackage{graphicx}
\usepackage{cite}
\usepackage{xcolor}
\usepackage{dsfont}
\usepackage{hyperref}
\usepackage{enumitem}
\newcommand{\subfiguretitle}[1]{{\scriptsize{#1}} \\}
\newcommand{\R}{\mathbb{R}}                                      % real numbers
\newcommand{\innerprod}[2]{\left\langle #1,\, #2 \right\rangle}  % scalar product
\newcommand{\ts}{\hspace*{0.1em}}                                % thin space
\providecommand{\abs}[1]{\left\lvert #1 \right\rvert}            % absolute value
\providecommand{\norm}[1]{\left\lVert #1 \right\rVert}           % norm
\newcommand\xqed[1]{\leavevmode\unskip\penalty9999 \hbox{}\nobreak\hfill \quad\hbox{#1}}
\newcommand{\exampleSymbol}{\xqed{$\triangle$}}

\DeclareMathOperator{\mspan}{span}

\newtheorem{theorem}{Theorem}[section]

\newtheorem{lemma}[theorem]{Lemma}
\newtheorem{proposition}[theorem]{Proposition}
\newtheorem{definition}[theorem]{Definition}
\newtheorem{example}[theorem]{Example}
\newtheorem{assumption}[theorem]{Assumption}
\newtheorem{remark}[theorem]{Remark}
\newtheorem{talgorithm}[theorem]{Algorithm}

\begin{document}

\title{Learning graphons from data: Random walks, \\ transfer operators, and spectral clustering}

\author{Stefan Klus and Jason J. Bramburger
\thanks{S. Klus is with the School of Mathematical \& Computer Sciences,  Heriot--Watt University, Edinburgh,UK}
\thanks{J.J. Bramburger is with the Department of Mathematics and Statistics, Concordia University, Montr\'eal, Canada.}
}

% \author{\IEEEauthorblockN{Stefan Klus} \\
% \IEEEauthorblockA{\textit{School of Mathematical \& Computer Sciences } \\
% \textit{Heriot--Watt University}\\
% Edinburgh, UK \\
% S.Klus@hw.ac.uk }\\
% \and
% \IEEEauthorblockN{Jason J. Bramburger} \\
% \IEEEauthorblockA{\textit{Deptartment of Mathematics and Statistics} \\
% \textit{Concordia University}\\
% Montreal, Canada \\
% jason.bramburger@concordia.ca}

        % <-this % stops a space
%\thanks{}% <-this % stops a space



\maketitle

\begin{abstract}
Many signals evolve in time as a stochastic process, randomly switching between states over discretely sampled time points. Here we make an explicit link between the underlying stochastic process of a signal that can take on a bounded continuum of values and a random walk process on a graphon. Graphons are infinite-dimensional objects that represent the limit of convergent sequences of graphs whose size tends to infinity. We introduce transfer operators, such as Koopman and Perron--Frobenius operators, associated with random walk processes on graphons and then illustrate how these operators can be estimated from signal data and how their eigenvalues and eigenfunctions can be used for detecting clusters, thereby extending conventional spectral clustering methods from graphs to graphons. Furthermore, we show that it is also possible to reconstruct transition probability densities and, if the random walk process is reversible, the graphon itself using only the signal. The resulting data-driven methods are applied to a variety of synthetic and real-world signals, including daily average temperatures and stock index values.
\end{abstract}


\section{Introduction}

Many signals in the real world that evolve in time can be modeled as a stochastic process with the signal randomly jumping from one state to another as time proceeds. When the signal can only exhibit a finite number of possible states, one can interpret the evolution of the signal as a random walk on a graph with vertices representing the states of the signal and edge weights giving way to the transition probabilities from one state to another. In particular, one arrives at a Markov chain representation of the signal that can be estimated using only the signal data. However, many realistic signals can take on a continuum of values, and so the goal of this work is to present a framework for modeling continuous-space stochastic signals and to identify metastable and coherent sets via clustering techniques.

We present a data-driven method to learn the discrete-time transition probabilities of stochastic signals evolving in continuous space, which can be regarded as a generalization of the discrete space case considered in \cite{KT24, KD24}. The underlying theory is developed by evoking the concept of a \emph{graphon}, which can be defined as the limit of sequences of dense networks that grow without bound~\cite{LS06, lovasz2012large, janson2013graphons}. As recently shown in \cite{PLC21}, graphons provide a well-developed framework for extending the concepts of random walks on finite graphs to stochastic processes evolving in continuous space. For example, random walks on graphs can be used to measure the centrality of vertices, and these concepts can also be extended to graphons \cite{APSS20}. Our goal is to identify transition probabilities, clusters, and the graphon itself from random walk data. Graphons are now finding extensive application in applied mathematics and engineering to perform signal processing on large networks exhibiting similar structures \cite{morency2021graphon, RCR21, levie2023graphon} as well as providing theoretical guarantees on the stability and transferability of graph neural networks \cite{ruiz2020graphon, keriven2020convergence, neuman2023transferability}.

Graphons originate in the theory of dense graph limits and exchangeable random graph models, where they provide a canonical representation of large-network asymptotics; see, e.g., \cite{lovasz2012large} for a comprehensive treatment and \cite{borgs2008convergent} for foundational results on convergent dense graph sequences. Beyond pure graph theory, graphons have also played an important role in systems and control, where they serve as limit objects for large-scale networked dynamical systems and mean-field control problems. Representative contributions include graphon game formulations for network interactions \cite{parise2023graphon} and control of large-scale linear networks via graphon limits \cite{gao2019graphon}. Recent work has also explored machine learning approaches for representing and learning graphons from data, including neural implicit representations of graphons \cite{XMW23}.

In contrast to these directions, the present work adopts an operator-theoretic perspective focused on stochastic processes defined directly on graphons and the data-driven approximation of their associated transfer operators. We will show that graphons allow us to define transfer operators associated to the stochastic process that governs the signal, thereby moving to a linear and deterministic, but infinite-dimensional representation of the underlying system. These transfer operators include the Perron--Frobenius operator that governs the evolution of probability densities and the Koopman operator that propagates scalar functions on the state-space (in expectation) \cite{LaMa94, DJ99, Mezic05}. Discrete counterparts of transfer operators associated with random walks on graphs were defined in \cite{KT24, KD24} to highlight relationships with graph Laplacians and to derive novel spectral clustering algorithms. Since graphons can be regarded as graphs with an uncountable number of nodes, a major contribution of this paper is to extend transfer operators to the graphon setting. The resulting operators share many similarities with transfer operators for continuous dynamical systems governed by stochastic differential equations. This allows us to define spectral clustering for symmetric (i.e., undirected) graphons in terms of \emph{metastability} \cite{Davies82a, SS13}. Metastability implies that the state space can be partitioned into disjoint sets (forming the clusters) in such a way that transitions between these sets are rare events. That is, a random walker will on average spend a long time within a cluster before it moves to a different cluster. Furthermore, using the notion of \emph{coherence} \cite{FrSaMo10, Froyland13}, a generalization of metastability to non-reversible processes, we can also detect clusters in asymmetric (i.e., directed) graphons. The main contributions of this paper are summarized as follows:
\begin{enumerate}[label=\roman*)]
\item We define transfer operators for graphons and derive spectral clustering methods.
\item Furthermore, we show how graphons can be estimated from random walk data.
\item We illustrate the results on both benchmark problems and real-world datasets.
\end{enumerate}

We emphasize that the stochastic processes considered in this work are discrete-time Markov chains defined directly on a continuum node set. That is, the graphon provides a model on $ [0, 1] $ equipped with the Lebesgue measure, and the associated random walk evolves according to a transition density on this continuous state space. Our framework therefore does not proceed by taking limits of finite graphs; rather, the infinite (continuum) node set is the primary object of study from the outset. Finite data and finite-dimensional approximations enter only through trajectory observations and Galerkin projections of the associated transfer operators.
Our approach is thus different from other estimators that learn a graphon from (sequences of) finite graphs, see, e.g., \cite{CA14, XMW23}.

We will introduce graphons, random walks, and the required notation in Section~\ref{sec:Graphons and random walks}. Transfer operators associated with symmetric graphons will be defined and analyzed in Section~\ref{sec:Transfer operators of symmetric graphons}. In particular, we employ \emph{extended dynamic mode decomposition} (EDMD) \cite{WKR15, KKS16} to estimate transfer operators and their spectral decompositions (and in turn the graphon) from data. Following similar work in~\cite{KT24}, we show that the eigenfunctions associated with the largest eigenvalues of the Koopman operator can be used for spectral clustering of undirected graphons. In Section~\ref{sec:Extension to asymmetric graphons}, we further extend this work to directed graphons and show how in this case singular functions of associated transfer operators can be used to detect clusters. Open problems and future work will be discussed in Section~\ref{sec:Discussion}.

%%%%%%%%%%%%%%%%%%%%%%%%%%%%%%%%%%%%%%%%%%%%%%%%%%%%%%%%%%%%%%%%%%%%%%%%%%%%%%%%%%%%%%%%%%
\section{Graphons and random walks}
\label{sec:Graphons and random walks}

To understand random walks in continuous space, we adopt the language and notation of graphons. Graphons arise naturally as the limit of growing sequences of graphs and as a rule for generating finite graphs on an arbitrary number of vertices. Much of this theory can be found in the book~\cite{Lov93}, while here we only present what is relevant to our results.

%%%%%%%%%%%%%%%%%%%%%%%%%%%%%%%%%%%%%%%%%%%%%%%%%%%%%%%%%%%%%%%%%%%%%%%%%%%%%%%%%%%%%%%%%%
\subsection{Graphons}

To begin, a \textbf{graphon} is a Lebesgue-measurable function $ w \colon [0,1] \times [0, 1] \to [0, 1] $. The function $ w $ can be understood to represent the weight of an edge between the continuum of vertices in the graph represented by the values in $ [0, 1] $. In particular, an edge is present between vertices $ x, y \in [0,1] $ if and only if $ w(x, y) > 0 $. A graphon is said to be \textbf{symmetric} or \textbf{undirected} if $ w(x, y) = w(y, x) $ for every pair $ x, y \in [0,1] $. Otherwise, the graphon is said to be \textbf{asymmetric} or \textbf{directed}.

Boundedness of the graphon implies that it belongs to $L^p([0,1]^2)$ for every $p \in [1,\infty]$. Moreover, a graphon can be used to define a kernel of a Hilbert--Schmidt integral operator of the form $\mathcal{W} f(x) = \int_0^1 w(x, y) \ts f(y) \ts \mathrm{d}y$.
This operator induced by the graphon is an infinite-dimensional version of considering the weighted adjacency matrix of a graph as an operator on a finite-dimensional Euclidean space. The operator-theoretic interpretation of the graphon allows us to take the spectrum of the graphon as the spectrum of $ \mathcal{W} $, analogous with the spectrum of a graph being the eigenvalues and eigenvectors of its weighted adjacency matrix.

\begin{definition}[Connectedness]
A graphon $ w $ is called \textbf{connected} if
\begin{equation*}
    \int_{A} \int_{A^c} w(x, y) \ts \mathrm{d}x \ts \mathrm{d}y > 0
\end{equation*}
for all sets $ A \subseteq [0,1] $ with Lebesgue measure $ 0 < \operatorname{vol}(A) < 1 $, where $ A^c = [0,1] \setminus A$ denotes the complement of $ A $ in $ [0, 1] $.
\end{definition}

Connectedness guarantees that edges between any subset of vertices $ A $ in $ [0, 1] $ and its complement $ A^c $ exist. The notion of \emph{strong connectivity} has been extended to graphons as well, see, e.g., \cite{BPS22}.

\begin{remark}
The presentation of graphons here has restricted the vertices to belong to the interval $ [0, 1] $ equipped with the Lebesgue measure. This is the standard convention for graphons, but we note that everything can be generalized to other probability spaces~\cite{LS06, janson2013graphons}. That is, given a probability space $ (\mathcal{X}, \Sigma, \mu) $ we may consider a graphon $ w \colon \mathcal{X} \times \mathcal{X} \to [0, 1] $ to be any $ (\Sigma \times \Sigma) $-measurable function, thus having vertices belonging to the space $ \mathcal{X} $. Throughout the theoretical work that follows we will continue with $\mathcal{X} = [0,1]$ and $\mu$ the Lebesgue measure to maintain a consistent and tidy presentation, while simply remarking that such a generalization is always possible without changing the results that follow.
\end{remark}

%%%%%%%%%%%%%%%%%%%%%%%%%%%%%%%%%%%%%%%%%%%%%%%%%%%%%%%%%%%%%%%%%%%%%%%%%%%%%%%%%%%%%%%%%%
\subsection{Random walks}

We now extend the concept of a random walk on a finite graph to a random walk in the continuous space $ [0, 1] $ using graphons. First, we require the definition of in- and out-degrees.

\begin{definition}[In- and out-degree functions] \label{def:degree}
We define the \textbf{in-degree function} $ d_\text{in} \colon [0,1] \to [0, 1] $ and \textbf{out-degree function} $ d_\text{out} \colon [0,1] \to [0, 1] $ by
\begin{equation*}
    d_\text{in}(x) = \int_0^1 w(y, x) \ts \mathrm{d}y \quad \mathrm{and} \quad
    d_\text{out}(x) = \int_0^1 w(x, y) \ts \mathrm{d}y.
\end{equation*}
For symmetric graphons, we define $ d(x) = d_\text{in}(x) = d_\text{out}(x) $.
\end{definition}

Connectedness of a graphon only implies $ d_\text{out}(x) > 0 $ for almost all $ x \in [0,1] $. However, we will make the assumption that $ d_\text{out}(x) $ is nonzero for all $x \in [0,1]$.

\begin{assumption} \label{assumpt:degree}
For any graphon $ w $ considered herein, there exists $ d_0 > 0 $ such that $ d_\text{out}(x) \geq d_0 $ for all $ x \in [0, 1] $.
\end{assumption}

Assumption~\ref{assumpt:degree} ensures that the invariant density of a random walk on a graphon is nonzero everywhere. This in turn eases much of the analysis that follows.

\begin{definition}[Transition density function] \label{def:transition density function}
Using the out-degree function, we construct the \emph{transition density function} $ p \colon [0,1] \times [0,1] \to [0,\infty) $ by $p(x, y) = {w(x, y)}/{d_\text{out}(x)}$
\end{definition}

That is, $ p(x, \,\cdot\,) $ is the density function describing the probability of going from $ x $ to any other point $ y \in [0,1] $. From the nonnegativity of the graphon $ w $ we immediately have that $ p(x, y) \ge 0 $, while the definition of $ d_\text{out} $ gives that $ \int_0^1 p(x, y) \ts \mathrm{d}y =  1 $. We can now use the transition density function to define a discrete-time random walk process on $ [0, 1] $ as follows.

\begin{definition}[Random walk]
Given the position $ x^{(k)} \in [0,1] $ of the random walker, we sample the new location $ x^{(k+1)} \sim p(x^{(k)}, \,\cdot\,) $.
\end{definition}

\begin{figure*}[t]
    \centering
    \begin{minipage}[t]{0.28\linewidth}
        \centering
        \subfiguretitle{(a)}
        \includegraphics[height=3.5cm]{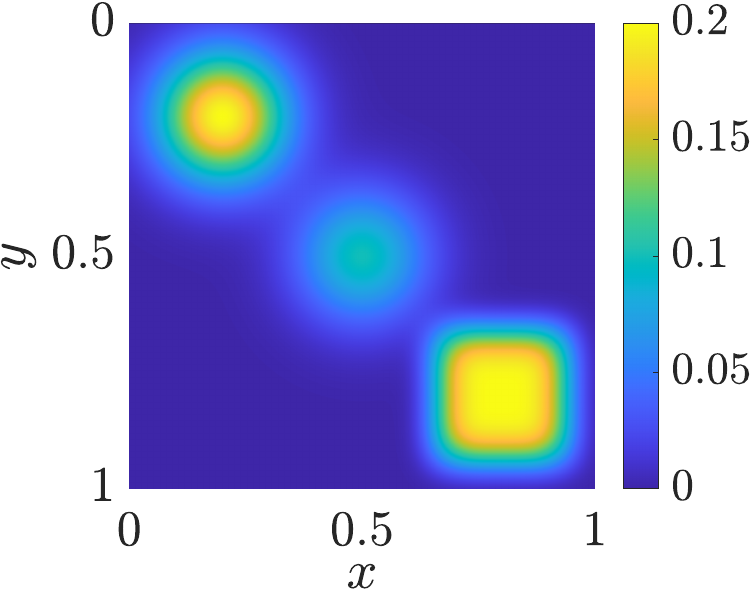}
    \end{minipage}
    \begin{minipage}[t]{0.28\linewidth}
        \centering
        \subfiguretitle{(b)}
        \includegraphics[height=3.5cm]{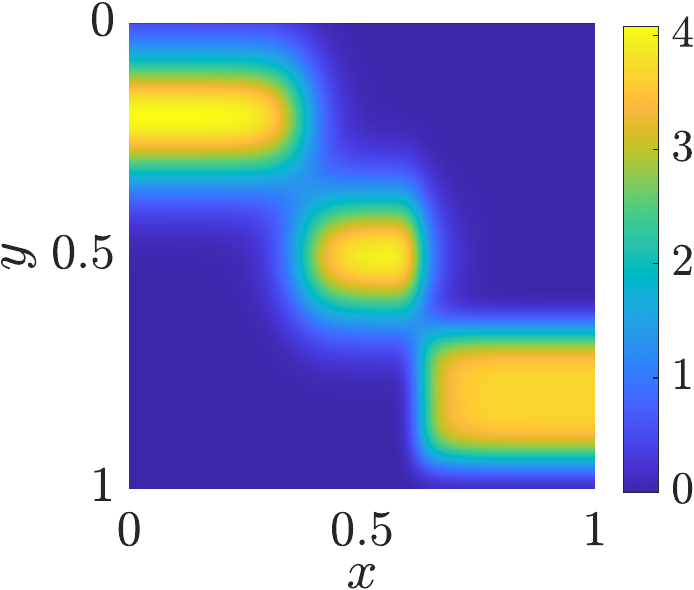}
    \end{minipage}
    \begin{minipage}[t]{0.30\linewidth}
        \centering
        \subfiguretitle{(c)}
        \vspace*{0.1ex}
        \includegraphics[height=3.59cm]{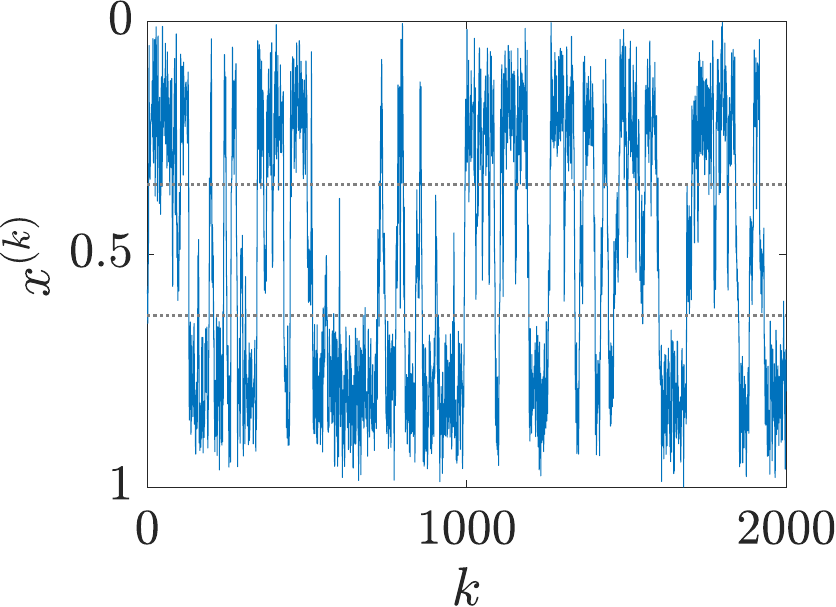}
    \end{minipage} \\[1ex]
    \begin{minipage}[t]{0.28\linewidth}
        \centering
        \subfiguretitle{(d)}
        \includegraphics[height=3.5cm]{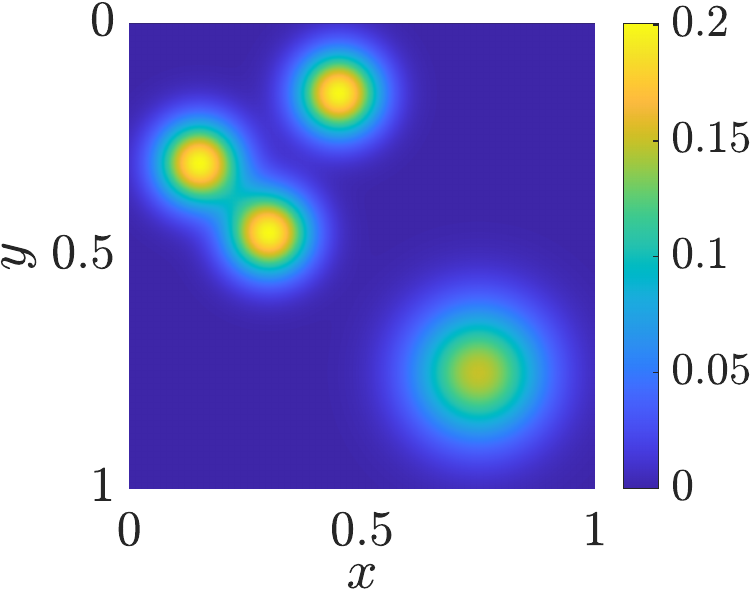}
    \end{minipage}
    \begin{minipage}[t]{0.28\linewidth}
        \centering
        \subfiguretitle{(e)}
        \includegraphics[height=3.5cm]{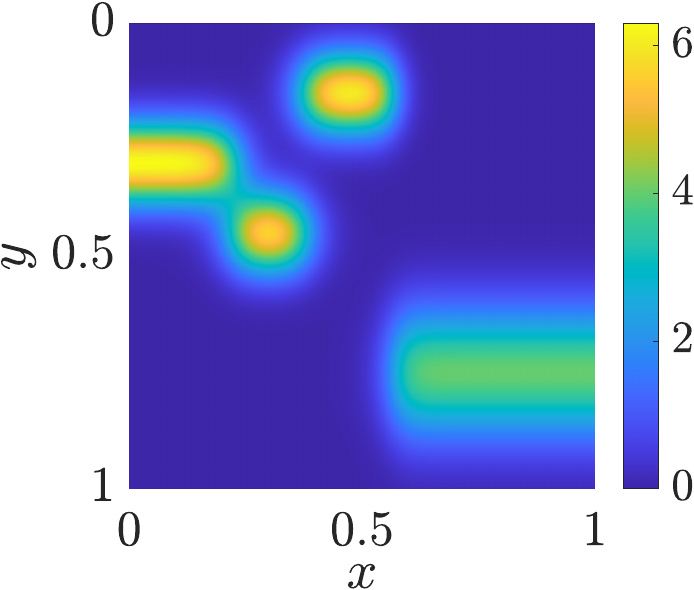}
    \end{minipage}
    \begin{minipage}[t]{0.30\linewidth}
        \centering
        \subfiguretitle{(f)}
        \vspace*{0.1ex}
        \includegraphics[height=3.59cm]{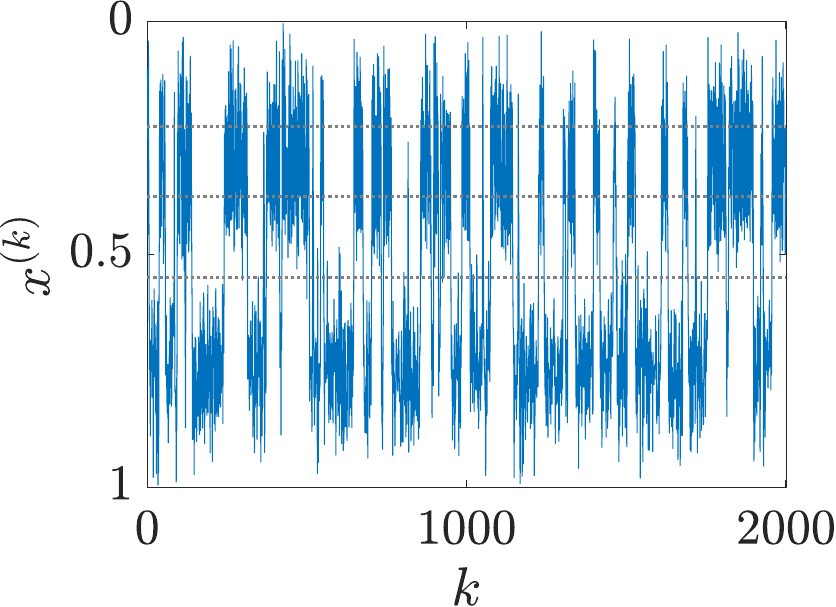}
    \end{minipage}
    \caption{(a)~The symmetric graphon with three peaks at $ x_1 = y_1 = 0.2 $, $ x_2 = y_2 = 0.5 $, and $ x_3 = y_3 = 0.8 $. The middle peak is lower than the other two. (b)~The corresponding transition probabilities $ p $. (c)~One long random walk comprising $ 2000 $ steps. The dotted gray lines mark the boundaries between the peaks. (d)~The asymmetric graphon with four peaks at $ x_1 = 0.15 $, $ y_1 = 0.3 $, $ x_2 = 0.3 $, $ y_2 = 0.45 $, $ x_3 = 0.45 $, $ y_3 = 0.15 $, and $ x_4 = 0.75 $, $ y_4 = 0.75 $. (e)~The corresponding transition probabilities $ p $. (f)~One long random walk comprising $ 2000 $ steps. A random walker typically quickly moves from the first to the second peak, then to the third, and back to the first due to the cyclic structure, while it might be trapped within the fourth cluster for a comparatively long time.}
    \label{fig:guiding examples}
\end{figure*}

This induces a non-deterministic discrete-time dynamical system of the form $x^{(k+1)} = \Theta\big(x^{(k)}\big)$ and can be viewed as a \textbf{Markov chain} defined on the continuous state-space $ [0, 1] $. In order to sample from the distribution $ p\big(x^{(k)}, \,\cdot\,\big) $, we can, for instance, use rejection sampling or inverse transform sampling, see, e.g., \cite{Bishop06} for an overview of sampling methods.

\begin{example} \label{ex:guiding examples}
In Figure~\ref{fig:guiding examples} we present random walks generated by the symmetric graphon
\begin{equation*}
    \begin{split}
         w(x, y) &=  0.2 \ts e^{-\frac{(x - 0.2)^2 + (y - 0.2)^2}{0.02}}
             + 0.1 \ts e^{-\frac{(x - 0.5)^2 + (y - 0.5)^2}{0.02}}\\ &+ 0.2 \ts e^{-\frac{(x - 0.8)^4 + (y - 0.8)^4}{0.0005}},
    \end{split}
\end{equation*}
and the asymmetric graphon
\begin{equation*}
\begin{split}
    w(x, &y) = 0.2 \ts e^{-\frac{(x - 0.15)^2 + (y - 0.3)^2}{0.008}}
         + 0.2 \ts e^{-\frac{(x - 0.3)^2 + (y - 0.45)^2}{0.008}} \\
         &~+ 0.2 \ts e^{-\frac{(x - 0.45)^2 + (y - 0.15)^2}{0.008}}
         + 0.15 \ts e^{-\frac{(x - 0.75)^2 + (y - 0.75)^2}{0.02}}.
 \end{split}
\end{equation*}
In the symmetric case the random walk process exhibits metastable behavior, in that if a random walker starts close to one of the three peaks, the probability to stay in the vicinity of the peak is large. In the asymmetric case only the rightmost peak is metastable, while the three non-diagonal peaks, roughly speaking, form a cycle of length three. \exampleSymbol
\end{example}

%%%%%%%%%%%%%%%%%%%%%%%%%%%%%%%%%%%%%%%%%%%%%%%%%%%%%%%%%%%%%%%%%%%%%%%%%%%%%%%%%%%%%%%%%%
\section{Transfer operators of symmetric graphons}
\label{sec:Transfer operators of symmetric graphons}

In this section, we focus exclusively on symmetric graphons. We first define transfer operators, analyze their eigenvalues and eigenfunctions, and show how they can not only be used for spectral clustering, but also for reconstructing transition densities and the graphon itself. The reason for considering symmetric graphons separately is that the associated random walk process is reversible, which implies that the resulting transfer operators are self-adjoint and thus have a real-valued spectrum. Much of this work will be extended to asymmetric graphons in Section~\ref{sec:Extension to asymmetric graphons} below.

%%%%%%%%%%%%%%%%%%%%%%%%%%%%%%%%%%%%%%%%%%%%%%%%%%%%%%%%%%%%%%%%%%%%%%%%%%%%%%%%%%%%%%%%%%
\subsection{Transfer operators}

In what follows, let $ L_\mu^p $ be the space of (equivalence classes of) functions such that
\begin{equation*}
    \norm{f}_{L_\mu^p} := \bigg(\int_0^1 \abs{f(x)}^p \mu(x) \ts \mathrm{d} x\bigg)^\frac{1}{p} < \infty,
\end{equation*}
where $ \mu $ is a probability density. The corresponding unweighted spaces will be denoted by $ L^p $. For $ p = 2 $, we obtain a Hilbert space with inner product
\begin{equation*}
    \innerprod{f}{g}_\mu = \int_0^1 f(x) \ts g(x) \ts \mu(x) \ts \mathrm{d}x.
\end{equation*}
The unweighted inner product is simply denoted by $ \innerprod{\cdot}{\cdot} $. Of particular interest is the probability density
\begin{equation} \label{eq:invariant density}
    \pi(x) = \frac{1}{Z} \ts d(x), \quad \text{with } Z = \int_0^1 d(x) \ts \mathrm{d}x.
\end{equation}
We will show that $ \pi $ is the uniquely defined invariant density below. Notice that Assumption~\ref{assumpt:degree} gives that there exists a $ d_0 > 0 $ such that $ \pi(x) $ is both bounded away from zero and bounded from uniformly over $ x \in [0, 1] $.

\begin{definition}[Perron--Frobenius and Koopman operators]
Given a symmetric graphon $ w $ with transition density function~$ p $.
\begin{enumerate}[label=\roman*)]
\item The \emph{Perron--Frobenius operator} $ \mathcal{P} \colon L_{\frac{1}{\pi}}^2 \to L_{\frac{1}{\pi}}^2 $ is defined by
\begin{equation*}
    \mathcal{P} \rho(x) = \int_0^1 p(y, x) \ts \rho(y) \ts \mathrm{d}y.
\end{equation*}
\item The Koopman operator $ \mathcal{K} \colon L_{\pi}^2 \to L_{\pi}^2 $ is defined by
\begin{equation*}
    \mathcal{K} f(x) = \int_0^1 p(x, y) \ts f(y) \ts \mathrm{d}y = \mathbb{E}[f(\Theta(x))].
\end{equation*}
\end{enumerate}
\end{definition}

The Perron--Frobenius operator describes the evolution of probability densities. Often one finds that $ \mathcal{P} $ and $ \mathcal{K} $ are considered as operators on $ L^1 $ and $ L^\infty $, respectively, but these operators are well-defined on the appropriately reweighted Hilbert spaces defined above \cite{KWNS18}.

\begin{remark}
Similar operators were also derived in \cite{PLC21} by considering the continuum limit of discrete- and continuous-time node-centric random walks on graphs. While they focus mostly on ensembles of random walkers defined by a probability density (and its evolution), we also consider individual random walks. This allows us to estimate transfer operators and their eigenvalues and eigenfunctions, which can then for example be used for spectral clustering, from trajectory data.
\end{remark}

\begin{lemma} \label{lem:invariant density}
Let $ w $ be a connected symmetric graphon, then the function $ \pi $, defined in \eqref{eq:invariant density}, is an invariant density, i.e., $ \mathcal{P} \pi = \pi $.
\end{lemma}

\begin{proof}
The proof is equivalent to the graph case found, for example, in~\cite{KD24}. Using the symmetry of $ w $, we have
\begin{equation*}
    \begin{split}
        \mathcal{P} \pi(x) &= \int_0^1 p(y, x) \ts \pi(y) \ts \mathrm{d}y = \frac{1}{Z} \int_0^1 \frac{w(y, x)}{d(y)} \ts d(y) \ts \mathrm{d}y\\ &= \frac{1}{Z} \int_0^1 w(x, y) \ts \mathrm{d}y= \frac{1}{Z}d(x) = \pi(x). \qedhere
    \end{split}
\end{equation*}
\end{proof}

It was shown in \cite{AR16} that for connected and symmetric graphons the invariant density is uniquely defined and, furthermore, that the random walk process is ergodic. This means that it almost surely holds that
\begin{equation*}
    \lim_{m \to \infty} \frac{1}{m} \sum_{k=1}^m f\big(x^{(k)}\big) = \int_0^1 f(x) \ts \pi(x) \ts \mathrm{d}x.
\end{equation*}

\begin{definition}[Reweighted Perron--Frobenius operator] \label{def:reweighted Perron-Frobenius operator}
Given the uniquely defined invariant density $ \pi $, we can further define the \emph{Perron--Frobenius operator with respect to the invariant density} $ \mathcal{T} \colon L_{\pi}^2 \to L_{\pi}^2 $, acting on functions $ u = \frac{\rho}{\pi} $, where $ \rho $ is a probability density, as
\begin{equation*}
    \mathcal{T} u(x) = \frac{1}{\pi(x)} \int_0^1 p(y, x) \ts \pi(y) \ts u(y) \ts \mathrm{d}y.
\end{equation*}
\end{definition}

That is, the operator $ \mathcal{T} $ propagates probability densities with respect to $ \pi $.

\begin{proposition} \label{pro:properties of transfer operators}
The transfer operators $ \mathcal{P} $, $ \mathcal{K} $, and $ \mathcal{T} $ have the following properties:
\begin{enumerate}[label=\roman*)]
\item We have $ \innerprod{\mathcal{P} \rho}{f} = \innerprod{\rho}{\mathcal{K} f} $, i.e., $ \mathcal{P} $ is the adjoint of $ \mathcal{K} $ with respect to $ \innerprod{\cdot}{\cdot} $.
\item Similarly, $ \innerprod{\mathcal{T} u}{f}_\pi = \innerprod{u}{\mathcal{K} f}_\pi $, i.e., $ \mathcal{T} $ is the adjoint of $ \mathcal{K} $ with respect to $ \innerprod{\cdot}{\cdot}_\pi $.
\item Let $ \rho $ be a probability density, then $ \mathcal{P} \rho(x) $ is also a probability density.
\item The Perron--Frobenius operator is a bounded Markov operator (on densities) with $ \norm{\mathcal{P}}_{L^1} = 1 $.
\item It holds that $ \mathcal{T} \mathds{1} = \mathds{1} $ and $ \mathcal{K} \mathds{1} = \mathds{1} $.
\item If $ p $ is continuous, then $ \ts\mathcal{P}\! $ and $ \ts\mathcal{K}\! $ are compact.
\item The spectra of $ \ts\mathcal{P}\! $ and $ \ts\mathcal{K}\! $ are contained in the unit disk.
\end{enumerate}
\end{proposition}

\begin{proof}
The properties of transfer operators for stochastic differential equations are well-understood and the corresponding proofs can be found, for example, in \cite{Froyland13, NoNu13, SS13, KKS16}. The derivations for graphons are similar. We include short proofs of the statements for the sake of completeness:
\begin{enumerate}[label=\roman*)]
\item This follows immediately from the definitions since
\begin{align*}
\innerprod{\mathcal{P} \rho}{f}
    &= \int_0^1 \mathcal{P} \rho(x) \ts f(x) \ts \mathrm{d}x\\
    &= \int_0^1\int_0^1 p(y, x) \ts \rho(y) \ts \mathrm{d}y \ts f(x) \ts \mathrm{d}x \\
    &= \int_0^1 \rho(y) \int_0^1 p(y, x) \ts f(x) \ts \mathrm{d}x \ts \mathrm{d}y\\
    &= \int_0^1 \rho(y) \ts \mathcal{K} f(y) \ts \mathrm{d}y = \innerprod{\rho}{\mathcal{K} f}.
\end{align*}
\item This is almost identical to the proof of property i).
\item First, $ \mathcal{P} \rho(x) \ge 0 $ since $ p(y, x) \ge 0 $ and $ \rho(x) \ge 0 $. Furthermore, we have
\begin{align*}
\int_0^1 \mathcal{P} \rho(x) \ts \mathrm{d}x &= \int_0^1\int_0^1 p(y, x) \ts \rho(y) \ts \mathrm{d}y \ts \mathrm{d}x \\
    &= \int_0^1 \!\!\! \underbrace{\int_0^1 p(y, x) \ts \mathrm{d}x}_{=1} \ts \rho(y) \ts \mathrm{d}y = 1.
\end{align*}
\item Let $ \norm{\rho}_{L^1} = 1 $, then
\begin{align*}
    \norm{\mathcal{P} \rho}_{L^1}
        &= \int_0^1 \bigg| \int_0^1 p(y, x) \rho(y) \ts \mathrm{d}y \bigg| \mathrm{d}x \\
        &\le \int_0^1\int_0^1 p(y, x) \abs{\rho(y)} \ts \mathrm{d}y \ts \mathrm{d}x \\
        &= \int_0^1\int_0^1 p(y, x) \ts \mathrm{d}x \abs{\rho(y)} \ts \mathrm{d}y = \norm{\rho}_{L^1} = 1,
\end{align*}
but choosing $ \rho \equiv 1 $ implies $ \norm{\mathcal{P} \rho}_{L^1} = 1 $ and $ \norm{\mathcal{P}}_{L^1} = 1 $, see also \cite{PLC21}. If $ \rho \ge 0 $, then $ \mathcal{P} \rho \ge 0 $ as shown above and $ \norm{\mathcal{P} \rho}_{L^1} = \norm{\rho}_{L^1} $, i.e., $ \mathcal{P} $ is Markov.
\item We have $ \mathcal{T} \mathds{1}(x) = \frac{1}{\pi(x)} \mathcal{P} \pi (x) = \mathds{1}(x) $. Similarly,
\begin{equation*}
    \mathcal{K} \mathds{1}(x) = \int_0^1 p(x, y) \ts \mathds{1}(y) \ts \mathrm{d}y = \mathds{1}(x)
\end{equation*}
since $ p(x, \cdot) $ is a probability density.
\item The properties of integral operators of this form have been analyzed in detail for unweighted $ L^2 $ spaces. Necessary and sufficient conditions for compactness can be found, e.g., in \cite{Ando62, Graham79}, while the compactness of the Perron--Frobenius operator for graphons on $ L^2 $ was also proven in~\cite{PLC21}. Then, since the invariant density $\pi$ is bounded away from zero and above uniformly in $ x \in [0,1] $, we have that $ L^2 $ and $ L^2_\pi $ are isometrically isomorphic via the map $L^2 \ni v \mapsto \pi^{-\frac{1}{2}}v \in L^2_\pi$. Thus, compactness in $ L^2 $ carries over to compactness in $ L^2_\pi $, proving the statement for~$ \mathcal{K} $. Similarly, $ \mathcal{P} $ is compact as an operator on $ L^2_\frac{1}{\pi} $ because $ L^2 $ and $ L^2_\frac{1}{\pi} $ are also isometrically isomorphic due again to the boundedness and uniform nonzero properties of $ \pi $.

\item Let $ \mathcal{P} \varphi = \lambda \ts \varphi $ with $\varphi \in L^2$. Then, since $L^2 \subset L^1$ and using property iv), we have $\norm{\varphi}_{L^1} \ge \norm{\mathcal{P} \varphi}_{L^1} = \abs{\lambda} \norm{\varphi}_{L^1}$, giving that $|\lambda| \leq 1$. Since $ \mathcal{K} = \mathcal{P}^* $, it holds that $ \sigma(\mathcal{K}) = \big\{\overline{\lambda} : \lambda \in \sigma(\mathcal{P}) \big\} $, where $ \overline{\lambda} $ denotes the complex conjugate of $ \lambda $. Then, to move to the weighted spaces $ L^2_\pi $ and $ L^2_\frac{1}{\pi} $ we again appeal to the fact that they are isometrically isomorphic to $ L^2 $, meaning the spectrum remains the same when posing the operators on these weighted spaces. \qedhere
\end{enumerate}
\end{proof}

%%%%%%%%%%%%%%%%%%%%%%%%%%%%%%%%%%%%%%%%%%%%%%%%%%%%%%%%%%%%%%%%%%%%%%%%%%%%%%%%%%%%%%%%%%
\subsection{Reversibility}

We begin by recalling the following definition, adapted for the continuous-space random walks herein.

\begin{definition}[Reversibility]
A process is called \emph{reversible} if there exists a probability density $ \pi $ that satisfies the \emph{detailed balance condition}
\begin{equation*}
    p(x, y) \ts \pi(x) = p(y, x) \ts \pi(y) ~~ \forall x, y \in [0,1].
\end{equation*}
\end{definition}

\begin{lemma}
The random-walk process defined on a symmetric graphon is reversible.
\end{lemma}

\begin{proof}
Since $ w(x, y) = w(y, x) $ for all $x,y \in [0,1]$, it holds that
\begin{equation*}
    \begin{split}
        p(x, y) \ts \pi(x) &= \frac{1}{Z} \frac{w(x, y)}{d(x)} d(x) \\ &= \frac{1}{Z} \frac{w(y, x)}{d(y)} d(y) = p(y, x) \ts \pi(y). \qedhere
    \end{split}
\end{equation*}
\end{proof}

This result implies that $ \mathcal{T} = \mathcal{K} $ since
\begin{equation*}
    \begin{split}
        \mathcal{T} u(x)
        &= \frac{1}{\pi(x)} \int_0^1 p(y, x) \ts \pi(y) \ts u(y) \ts \mathrm{d}y \\
        &= \int_0^1 p(x, y) \ts u(y) \ts \mathrm{d}y = \mathcal{K} u(x).
    \end{split}
\end{equation*}
Furthermore, $ \mathcal{P} $ and $ \mathcal{K} $ are then self-adjoint operators with respect to appropriately weighted inner products.

\begin{lemma} \label{lem:self-adjoint}
Given a symmetric graphon $ w $, let $ \pi $ be the invariant density defined in \eqref{eq:invariant density}, then
\begin{align*}
    \innerprod{\mathcal{P} \rho}{\sigma}_{\frac{1}{\pi}} = \innerprod{\rho}{\mathcal{P} \sigma}_{\frac{1}{\pi}}
    \quad \text{and} \quad
    \innerprod{\mathcal{K} f}{g}_\pi = \innerprod{f}{\mathcal{K}g}_\pi.
\end{align*}
\end{lemma}

\begin{proof}
It holds that
\begin{align*}
    \innerprod{\mathcal{P} \rho}{\sigma}_{\frac{1}{\pi}}
        &= \int_0^1 \mathcal{P} \rho(x) \ts \sigma(x) \ts \frac{1}{\pi(x)} \ts \mathrm{d}x\\ &= \int_0^1\int_0^1 \frac{p(y, x)}{\pi(x)} \ts \rho(y) \ts \sigma(x) \ts \mathrm{d}x \ts \mathrm{d}y \\
        &= \int_0^1\int_0^1 \frac{p(x, y)}{\pi(y)} \ts \rho(y) \ts \sigma(x) \ts \mathrm{d}x \ts \mathrm{d}y \\
        &= \int_0^1 \rho(y) \mathcal{P} \sigma(y) \ts \frac{1}{\pi(y)} \ts \mathrm{d}y \\
        &= \innerprod{\rho}{\mathcal{P}\sigma}_{\frac{1}{\pi}},
\end{align*}
see also \cite{NoNu13}. The result for the Koopman operator follows immediately from Proposition~\ref{pro:properties of transfer operators}, property ii), and the fact that $ \mathcal{T} = \mathcal{K} $.
\end{proof}

%%%%%%%%%%%%%%%%%%%%%%%%%%%%%%%%%%%%%%%%%%%%%%%%%%%%%%%%%%%%%%%%%%%%%%%%%%%%%%%%%%%%%%%%%%
\subsection{Spectral decompositions}

Since the Perron--Frobenius and the Koopman operators are self-adjoint with respect to the appropriate weighted inner products given in Lemma~\ref{lem:self-adjoint}, this implies that their eigenvalues are real-valued. Moreover, the following lemma details a correspondence between the eigenvalues and eigenfunctions of these operators.

\begin{lemma} \label{lem:eigenfunction properties}
Given a symmetric graphon $w$, let $ \varphi_\ell $ be an eigenfunction of the associated Koopman operator, then $ \widehat{\varphi}_\ell = \pi \ts \varphi_\ell $ is an eigenfunction of the Perron--Frobenius operator, i.e.,
\begin{equation*}
    \mathcal{K} \varphi_\ell = \lambda_\ell \ts \varphi_\ell
    \implies
    \mathcal{P} \widehat{\varphi}_\ell = \lambda_\ell \ts \widehat{\varphi}_\ell.
\end{equation*}
\end{lemma}

\begin{proof}
We have
\begin{align*}
    \mathcal{P} \widehat{\varphi}_\ell(x)
        &= \int_0^1 p(y, x) \ts \pi(y) \ts \varphi_\ell(y) \ts \mathrm{d}y \\
        &= \int_0^1 p(x, y) \ts \pi(x) \ts \varphi_\ell(y) \ts \mathrm{d}y \\
        &= \pi(x) \ts \mathcal{K} \varphi_\ell(x) = \lambda_\ell \ts \pi(x) \ts \varphi_\ell(x) = \lambda_\ell \ts \widehat{\varphi}_\ell(x). \qedhere
\end{align*}
\end{proof}

Thus, using the well-known spectral decomposition of compact linear operators---see Appendix~\ref{app:spectral decompositions} for more details---, we can write
\begin{align*}
    \mathcal{P} \rho &= \sum_\ell \lambda_\ell (\widehat{\varphi}_\ell \otimes_{\frac{1}{\pi}} \widehat{\varphi}_\ell) \rho = \sum_\ell \lambda_\ell \innerprod{\widehat{\varphi}_\ell}{\rho}_{\frac{1}{\pi}} \widehat{\varphi}_\ell \\ &= \sum_\ell \lambda_\ell \innerprod{\varphi_\ell}{\rho} \widehat{\varphi}_\ell= \sum_\ell \lambda_\ell (\widehat{\varphi}_\ell \otimes \varphi_\ell) \rho, \\
    \mathcal{K} f &= \sum_\ell \lambda_\ell (\varphi_\ell \otimes_\pi \varphi_\ell) f = \sum_\ell \lambda_\ell \innerprod{\varphi_\ell}{f}_\pi \varphi_\ell\\ &= \sum_\ell \lambda_\ell \innerprod{\widehat{\varphi}_\ell}{f} \varphi_\ell = \sum_\ell \lambda_\ell (\varphi_\ell \otimes \widehat{\varphi}_\ell) f.
\end{align*}
The eigendecomposition also allows us to write the transition density function in terms of the eigenfunctions, i.e.,
\begin{equation*}
    p(x, y) = \sum_\ell \lambda_\ell \ts \varphi_\ell(x) \ts \widehat{\varphi}_\ell(y).
\end{equation*}
A similar decomposition of the transition probabilities associated with stochastic processes was derived in \cite{NC15}. Combining the definition of the transition density function and \eqref{eq:invariant density}, it follows that $ w(x, y) = d(x) \ts p(x, y)  = Z \ts \pi(x) \ts p(x, y) $ and consequently
\begin{equation*}
    w(x, y) = Z \sum_\ell \lambda_\ell \ts \widehat{\varphi}_\ell(x) \ts \widehat{\varphi}_\ell(y).
\end{equation*}
Using the eigenvalues and eigenfunctions of associated transfer operators, we can thus reconstruct the graphon up to a multiplicative factor. The non-uniqueness is due to the fact that the graphons $ w $ and $ c \ts w $ for a positive constant $ c $ give rise to the same transition densities and consequently also the same transfer operators.

\begin{remark} \label{rem:low-rank approximation}
Assuming there are $ r $ eigenvalues $ \lambda_1, \dots, \lambda_r $ close to 1, followed by a spectral gap so that $ \lambda_k \approx 0 $ for $ k \geq r+1 $, we then have
\begin{equation*}
    p(x, y) \approx \sum_{\ell=1}^r \lambda_\ell \ts \varphi_\ell(x) \ts \widehat{\varphi}_\ell(y), ~~
    w(x, y) \approx Z \sum_{\ell=1}^r \lambda_\ell \ts \widehat{\varphi}_\ell(x) \ts \widehat{\varphi}_\ell(y).
\end{equation*}
The number of dominant eigenvalues is related to the number of metastable sets or clusters. This will be illustrated in more detail below.
\end{remark}

We close this section with the following useful property on the boundedness and continuity of the eigenfunctions of $ \mathcal{K} $. Coupling this result with that of Lemma~\ref{lem:eigenfunction properties}, we obtain boundedness and continuity of the eigenfunctions of $ \mathcal{P} $.

\begin{lemma}
Under Assumption~\ref{assumpt:degree}, every eigenfunction of $ \mathcal{K} $ on $ L_\pi^2 $ having nonzero eigenvalue is bounded. Further, if the graphon $ w $ is such that for all $ a \in [0, 1] $ we have
\begin{equation} \label{eq:degree continuity}
    \lim_{x \to a} \int_0^1 \abs{w(x,y) - w(a,y)} \ts \mathrm{d}y = 0,
\end{equation}
then every eigenfunction of $ \mathcal{K} $ on $ L_\pi^2 $ having nonzero eigenvalue is continuous.
\end{lemma}

\begin{proof}
To begin, the Assumption~\ref{assumpt:degree} gives that there exists a $d_0 > 0$ such that $ d(x) \geq d_0 $ for all $ x \in [0, 1] $ and the uniform bound $ 0 \leq w(x,y) \leq 1 $ immediately gives that $ 0 \leq p(x,y) \leq d_0^{-1} $. Recall from the fact that $ \pi $ is nonzero and bounded above, $ L^2 $ and $ L^2_\pi $ are isometrically isomorphic and therefore proving the result for eigenfunctions of $ \mathcal{K} \colon L^2 \to L^2$ will prove it for $\mathcal{K} \colon L_\pi^2 \to L_\pi^2$ too.

Now, suppose $ \lambda_\ell $ is a nonzero eigenvalue of $ \mathcal{K} $ with eigenfunction $ \varphi_\ell \in L^2 $, i.e., $ \mathcal{K} \varphi_\ell(x) = \lambda_\ell \ts \varphi_\ell(x) $. Then, the Cauchy--Schwarz inequality gives that for each $ x \in [0, 1] $ we have
\begin{align*}
    \abs{\varphi_\ell(x)} &= \bigg|\frac{1}{\lambda_\ell} \mathcal{K} \varphi_\ell(x)\bigg| \leq \frac{1}{\abs{\lambda_\ell}} \int_0^1 \abs{p(x,y) \ts \varphi_\ell(y)} \ts \mathrm{d}y \\
        & \leq \frac{1}{\abs{\lambda_\ell}} \bigg(\int_0^1 |p(x,y)|^2 \mathrm{d}y\bigg)^\frac{1}{2}\bigg(\int_0^1 |\varphi_\ell(y)|^2 \mathrm{d}y\bigg)^\frac{1}{2} \\
        &\leq \frac{1}{d_0 \abs{\lambda_\ell}}\|\varphi_\ell\|_{L^2}.
\end{align*}
Thus, we have a bound on $ |\varphi_\ell(x)| $ that is independent of $x$, showing that the eigenfunction is bounded.

Now, let us further assume that \eqref{eq:degree continuity} holds. One can easily check that this implies that $ d $ is continuous in $x$. Rearranging the eigenfunction relationship $ \mathcal{K} \varphi_\ell(x) = \lambda_\ell \ts \varphi_\ell(x) $ as
\begin{equation*}
    \varphi_\ell(x) = \frac{1}{\lambda_\ell \ts d(x)} \int_0^1 w(x,y) \ts \varphi_\ell(y) \ts \mathrm{d}y
\end{equation*}
gives that $ \varphi_\ell $ is continuous so long as $ x \mapsto \int_0^1 w(x, y) \ts \varphi_\ell(y) \ts \mathrm{d}y $ is continuous since we have that $ \lambda_\ell \ts d(x) $ is continuous and bounded away from zero. Letting $ a \in [0, 1] $ and recalling that we have already shown that there exists $ M > 0 $ such that $ \abs{\varphi_\ell(x)} \le M $ for all $ x \in [0, 1] $, continuity of $ x \mapsto \int_0^1 w(x,y) \varphi_\ell(y) \ts \mathrm{d}y $ is a result of
\begin{equation*}
    \begin{split}
        \lim_{x \to a} &\bigg|\int_0^1 w(x, y) \ts \varphi(y) \ts \mathrm{d}y - \int_0^1 w(a, y) \ts \varphi_\ell(y) \ts \mathrm{d}y\bigg|\\ &\leq M \cdot \lim_{x \to a} \int_0^1 \abs{w(x, y) - w(a, y)} \ts \mathrm{d}y = 0,
    \end{split}
\end{equation*}
where we have used condition \eqref{eq:degree continuity}.
\end{proof}

It is important to note that nothing in the above lemma requires that the graphon be symmetric, and so these results still hold for asymmetric graphons as well, under Assumption~\ref{assumpt:degree}. While condition \eqref{eq:degree continuity} may be unintuitive, it is used to show that the degree function $ d $ is continuous. This condition is easily satisfied when the graphon $ w $ itself is continuous, but can hold for discontinuous graphons as well, see \cite{bramburger2024persistence}.

%%%%%%%%%%%%%%%%%%%%%%%%%%%%%%%%%%%%%%%%%%%%%%%%%%%%%%%%%%%%%%%%%%%%%%%%%%%%%%%%%%%%%%%%%%
\subsection{Learning transfer operators from data}

We now present a data-driven method for estimating transfer operators associated with graphons. Assuming we have random walk data of the form $ \{ x^{(1)}, x^{(2)}, \dots, x^{(m+1)} \} $ pertaining to a potentially unknown graphon, we define data pairs $ \big\{ (x^{(k)}, y^{(k)}) \big\}_{k=1}^m $, where $ y^{(k)} = x^{(k+1)} $. If $ m $ is large enough, the points $ x^{(k)} $ and $ y^{(k)} $ are approximately sampled from the invariant distribution~$ \pi $. In addition to the training data, we need to choose a set of basis functions $ \{ \phi_i \}_{i=1}^n $, termed a \emph{dictionary}. We then construct the vector-valued function $ \phi(x) = [\phi_1(x), \phi_2(x), \dots, \phi_n(x)]^\top $. The uncentered covariance and cross-covariance matrices $ \widetilde{C}_{xx} $ and $ \widetilde{C}_{xy} $, defined by
\begin{alignat*}{2}
    \widetilde{C}_{xx} &:= \frac{1}{m} \sum_{k=1}^m \phi(x^{(k)}) \otimes \phi(x^{(k)}) \\ &\underset{\scriptscriptstyle m \rightarrow \infty}{\longrightarrow} \int_0^1 \phi(x) \otimes \phi(x) \ts \pi(x) \ts \mathrm{d}x = C_{xx}, \\
    \widetilde{C}_{xy} &:= \frac{1}{m} \sum_{k=1}^m \phi(x^{(k)}) \otimes \phi(y^{(k)}) \\ &\underset{\scriptscriptstyle m\rightarrow \infty}{\longrightarrow}
    \int_0^1 \phi(x) \otimes \mathcal{K} \phi(x) \ts \pi(x) \ts \mathrm{d}x = C_{xy},
\end{alignat*}
converge in the infinite data limit to the matrices $ C_{xx} $ and $ C_{xy} $ required for the Galerkin approximation. That is, the matrix $ \widetilde{K} = \widetilde{C}_{xx}^{-1} \widetilde{C}_{xy} $ can be viewed as a data-driven Galerkin approximation of the Koopman operator with respect to the $ \pi $-weighted inner product, see Appendix~\ref{app:Galerkin approximation} for more details. This approach is called \emph{extended dynamic mode decomposition} (EDMD) \cite{WKR15, KKS16} and its convergence in the infinite data limit is guaranteed by results in \cite{KKS16, KoMe18, bramburger2024auxiliary}. The adjoint of $ \mathcal{K} $ with respect to the $ \pi $-weighted inner product is the the reweighted Perron--Frobenius operator~$ \mathcal{T} $, see Proposition~\ref{pro:properties of transfer operators}, which can be computed using Lemma~\ref{lem:Galerkin adjoint}, but for reversible systems, it holds that $ \mathcal{T} = \mathcal{K} $ as shown above. The eigenfunctions of the Perron--Frobenius operator $ \mathcal{P} $ can be computed with the aid of Lemma~\ref{lem:eigenfunction properties}, either using the true invariant density \eqref{eq:invariant density} or estimating it from trajectory data. This will be illustrated in more detail below. We are therefore now in a position to approximate eigenvalues and eigenfunctions of projected transfer operators, which allows us to reconstruct transition densities and the graphon itself and in particular also to apply spectral clustering methods.

\begin{talgorithm}[Data-driven spectral clustering algorithm for symmetric graphons] $ $
\begin{enumerate} \setlength{\itemsep}{0ex}
\item Given training data $ \{ (x^{(k)}, y^{(k)}) \}_{k=1}^m $ and the vector-valued function $ \phi $.
\item Compute the $r$ dominant eigenvalues $ \lambda_\ell $ and associated eigenvectors $ \xi_\ell $ of $ \widetilde{K} $.
\item Evaluate the resulting eigenfunctions $ \widetilde{\varphi}_\ell(x) := \xi_\ell^\top \phi(x) $ at all points $ x^{(k)} $.
\item Let $ s_i \in \R^{r} $ denote the (transposed) $ i $th row of
\begin{equation*}
    \boldsymbol{\widetilde{\varphi}{}} =
    \begin{bmatrix}
        \widetilde{\varphi}_1(x^{(1)}) & \dots & \widetilde{\varphi}_r(x^{(1)}) \\
        \vdots & \ddots & \vdots \\
        \widetilde{\varphi}_1(x^{(m)}) & \dots & \widetilde{\varphi}_r(x^{(m)}) \\
    \end{bmatrix}
    \in \R^{m \times r}.
\end{equation*}
\item Cluster the points $ \{ \ts s_i \ts \}_{i=1}^m $ into $ r $ clusters using, e.g., $ k $-means.
\end{enumerate}
\end{talgorithm}

The parameter $ r $ should be chosen in such a way that there is a spectral gap between $ \lambda_r $ and $ \lambda_{r+1} $. If no such spectral gap exists, this in general implies that there are no clearly separated clusters. The accuracy of the estimated eigenfunctions depends strongly on the basis functions and the graphon itself. While monomials typically lead to ill-conditioned matrices, indicator functions, Gaussian functions, or orthogonal polynomials in general work well. Choosing the right type and optimal number of basis functions is an open problem. In the last years, many dictionary learning methods have been proposed. The idea is to train a deep or shallow neural network whose output layer represents the basis functions. Loss functions are typically either based on the reconstruction error or some variational formulation~\cite{LDBK17, MPWN18, TLK25}. Convergence proofs and error bounds, which could be extended to the graphon case, can be found in \cite{KoMe18, ZZ23, LLLK24}. If the graphon $ w $ is known, it is of course also possible to directly compute the integrals required for the Galerkin approximation using, for instance, numerical quadrature rules instead of relying on Monte Carlo estimates.

\begin{remark}
We can define the random-walk normalized graphon Laplacian by $ \mathcal{L}_\text{rw} = \mathcal{I} - \mathcal{K} $. This is consistent with the definition of the random-walk normalized graph Laplacian in the finite-dimensional case. Assume that $ \mathcal{K} \varphi_\ell = \lambda_\ell \ts \varphi_\ell $, then $ \mathcal{L}_\text{rw} \ts \varphi_\ell = (\mathcal{I} - \mathcal{K}) \ts \varphi_\ell = (1 - \lambda_\ell) \ts \varphi_\ell $, i.e., the eigenfunctions of the random-walk normalized Laplacian are the eigenfunctions of the Koopman operator. Therefore, we see that clustering according to the Koopman eigenfunctions is equivalent to the more standard clustering with the Laplacian eigenfunctions.
\end{remark}

%%%%%%%%%%%%%%%%%%%%%%%%%%%%%%%%%%%%%%%%%%%%%%%%%%%%%%%%%%%%%%%%%%%%%%%%%%%%%%%%%%%%%%%%%%
\subsection{Numerical demonstrations}

Let us illustrate the clustering approach and the reconstruction of the graphon itself with the aid of examples.

%%%%%%%%%%%%%%%%%%%%%%%%%%%%%%%%%%%%%%%%%%%%%%%%%%%%%%%%%%%%%%%%%%%%%%%%%%%%%%%%%%%%%%%%%%
\subsubsection{Synthetic data}

\begin{figure*}
    \definecolor{matlab1}{RGB}{0, 114, 189}
    \definecolor{matlab2}{RGB}{217, 83, 25}
    \definecolor{matlab3}{RGB}{237, 177, 32}
    \newcommand{\cdash}[1]{\textcolor{#1}{\rule[0.5ex]{1em}{0.3ex}}}
    \centering
    \begin{minipage}[t]{0.28\linewidth}
        \centering
        \subfiguretitle{(a)}
        \vspace*{0.5ex}
        \includegraphics[height=3.5cm]{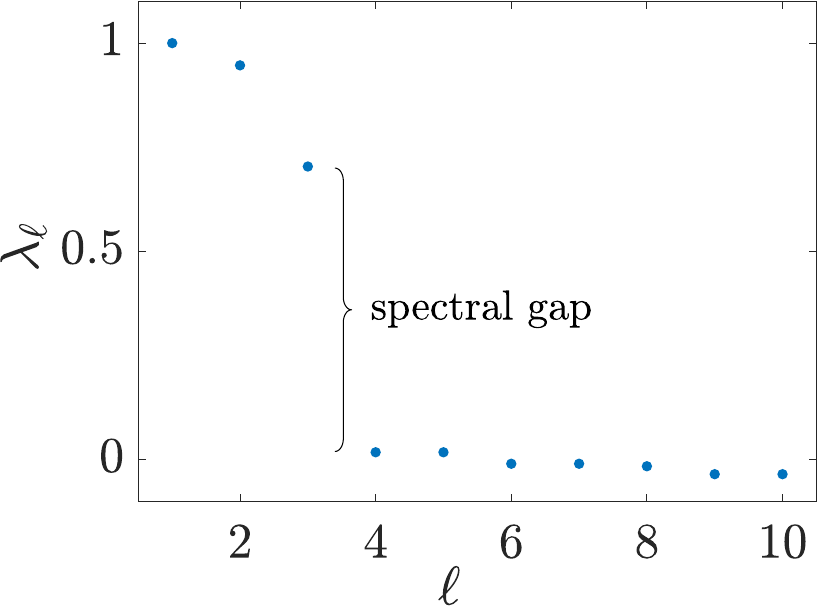}
    \end{minipage}
    \begin{minipage}[t]{0.28\linewidth}
        \centering
        \subfiguretitle{(b)}
        \vspace*{0.5ex}
        \includegraphics[height=3.48cm]{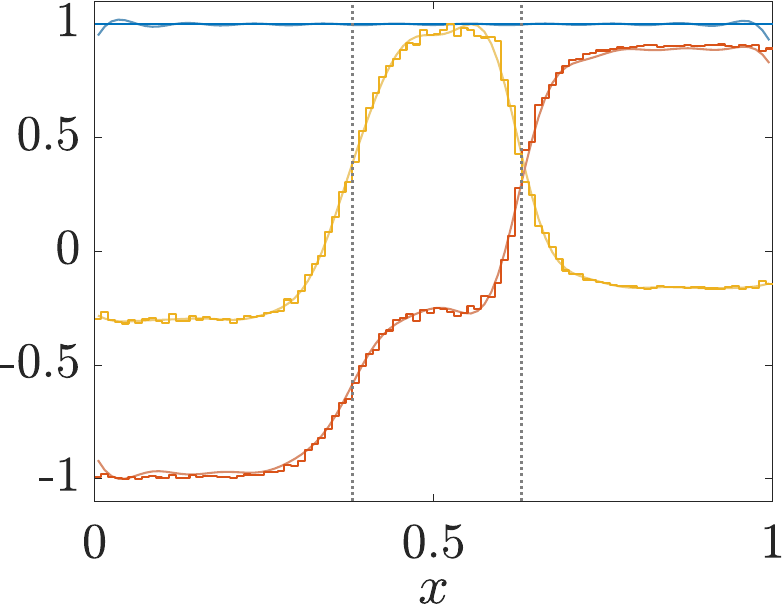}
    \end{minipage}
    \begin{minipage}[t]{0.28\linewidth}
        \centering
        \subfiguretitle{(c)}
        \vspace*{0.5ex}
        \includegraphics[height=3.5cm]{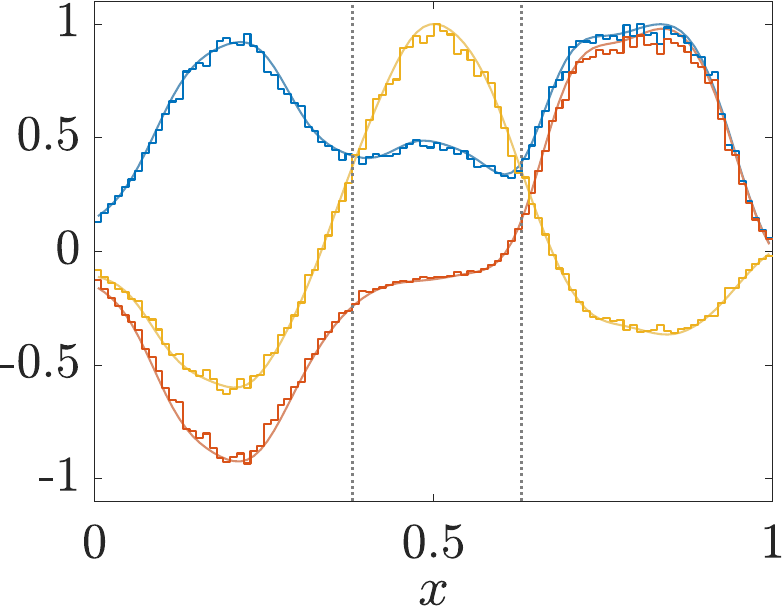}
    \end{minipage} \\[2ex]
    \begin{minipage}[t]{0.28\linewidth}
        \centering
        \subfiguretitle{(d)}
        \includegraphics[height=3.5cm]{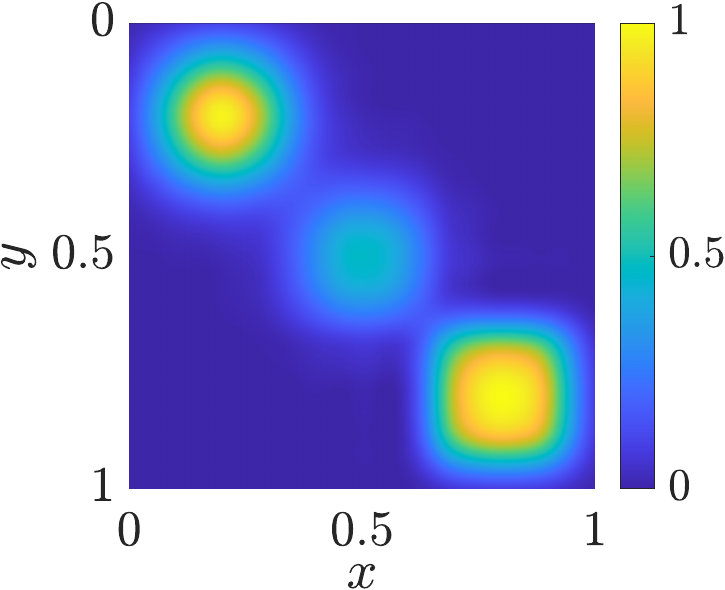}
    \end{minipage}
    \begin{minipage}[t]{0.28\linewidth}
        \centering
        \subfiguretitle{(e)}
        \includegraphics[height=3.5cm]{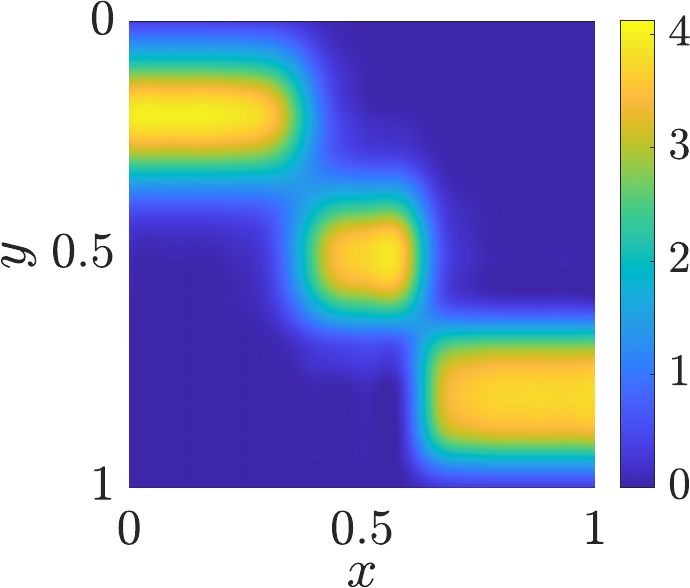}
    \end{minipage}
    \begin{minipage}[t]{0.28\linewidth}
        \centering
        \subfiguretitle{(f)}
        \includegraphics[height=3.5cm]{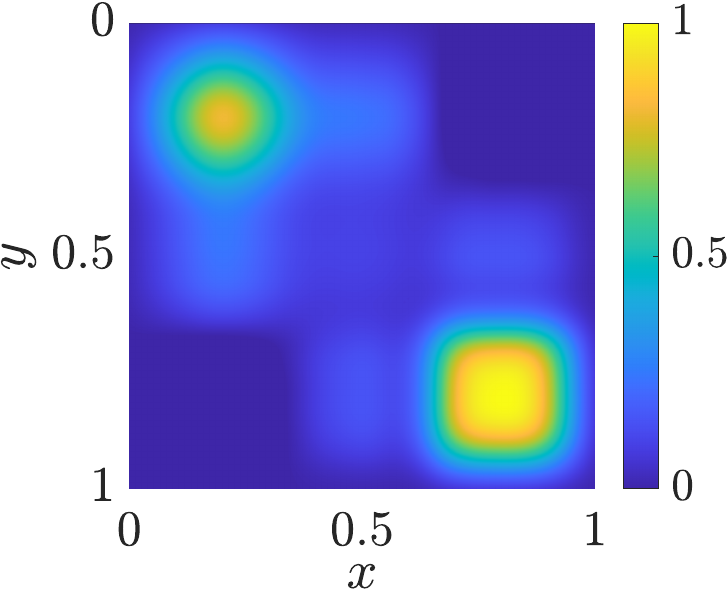}
    \end{minipage}
    \caption{(a)~First ten numerically computed eigenvalues of the Koopman and Perron--Frobenius operators. There is a large spectral gap between the third and fourth eigenvalue. (b)~Dominant eigenfunctions of the Koopman operator using 100 indicator functions (dark-colored) and 20 Gaussians (light-colored), where \cdash{matlab1} denotes the first, \cdash{matlab2} the second, and \cdash{matlab3} the third eigenfunction. The dotted gray lines separate the detected clusters. (c)~Numerically computed dominant eigenfunctions of the Perron--Frobenius operator. (d)~Resulting rank-3 approximation of $ w $ using the light-colored eigenfunctions. (e)~Rank-3 approximation of $ p $. (f)~Rank-2 approximation of $ w $ for comparison. The middle peak is missing and there are some artifacts.}
    \label{fig:triple-peak eigenfunctions}
\end{figure*}

As a first demonstration of the power of our data-driven spectral clustering algorithm, we apply it to the symmetric graphon defined in Example~\ref{ex:guiding examples}. We first generate a random walk of length $m = 20\ts000$ and then apply EDMD to estimate the dominant eigenfunctions of the Koopman operator using two different dictionaries:
\begin{enumerate}[label=\roman*)]
\item 100 indicator functions for an equipartition of $ [0, 1] $ into 100 intervals (i.e., Ulam's method);
\item 20 Gaussian functions with bandwidth $ \sigma = 0.05 $ centered at evenly-spaced points in $ [0, 1] $.
\end{enumerate}
The graphon structure leads to three dominant eigenvalues $ \lambda_1 = 1 $, $ \lambda_2 \approx 0.95 $, and $ \lambda_3 \approx 0.71 $, followed by a significant spectral gap, as shown in Figure~\ref{fig:triple-peak eigenfunctions}\ts(a). The subsequent eigenvalues after the gap are approximately zero. The estimated eigenfunctions of $ \mathcal{K} $ and the resulting clustering into three groups are shown in Figure~\ref{fig:triple-peak eigenfunctions}\ts(b). We compute the eigenfunctions of the Perron--Frobenius operator, shown in Figure~\ref{fig:triple-peak eigenfunctions}\ts(c), by evoking Lemma~\ref{lem:eigenfunction properties}, where we estimate the invariant density using kernel density estimation, see, e.g., \cite{Bishop06}. The first eigenfunction matches the analytically computed invariant density $ \pi $. Figures~\ref{fig:triple-peak eigenfunctions}\ts(d) and (e) show rank-3 approximations of the graphon $ w $ and transition density function $ p $, cf. Figure~\ref{fig:guiding examples}\ts(a) and (b), computed using the numerically approximated eigenfunctions, per Remark~\ref{rem:low-rank approximation}. Moreover, Figure~\ref{fig:triple-peak eigenfunctions}\ts(f) shows that the rank-2 reconstruction of $ w $ does not contain the middle peak and therefore misses critical information about the clusters. In order to measure how metastable the three detected clusters are, we estimate transition probabilities between them by counting the numbers of transitions contained in our training data. We define the transition matrix $ C $, where $ c_{ij} $ is the probability of going from cluster $ i $ to cluster $ j $ (in percent), and obtain
\begin{equation*}
    C =
    \begin{bmatrix} %[rrr]
       \mathbf{91.4} & 8.1 & 0.5 \\
       17.5 & \mathbf{70.2} & 12.3 \\
       0.4 & 4.4 & \mathbf{95.2}
    \end{bmatrix},
\end{equation*}
which shows that the middle cluster is as expected less metastable.

%%%%%%%%%%%%%%%%%%%%%%%%%%%%%%%%%%%%%%%%%%%%%%%%%%%%%%%%%%%%%%%%%%%%%%%%%%%%%%%%%%%%%%%%%%
\subsubsection{Daily average temperatures}

\begin{figure*}
    \definecolor{matlab1}{RGB}{0, 114, 189}
    \definecolor{matlab2}{RGB}{217, 83, 25}
    \definecolor{matlab3}{RGB}{237, 177, 32}
    \definecolor{matlab4}{RGB}{126, 47, 142}
    \newcommand{\cdash}[1]{\textcolor{#1}{\rule[0.5ex]{1em}{0.3ex}}}
    \centering
    \begin{minipage}[t]{0.28\linewidth}
        \centering
        \subfiguretitle{(a)}
        \vspace*{0.5ex}
        \includegraphics[height=3.25cm]{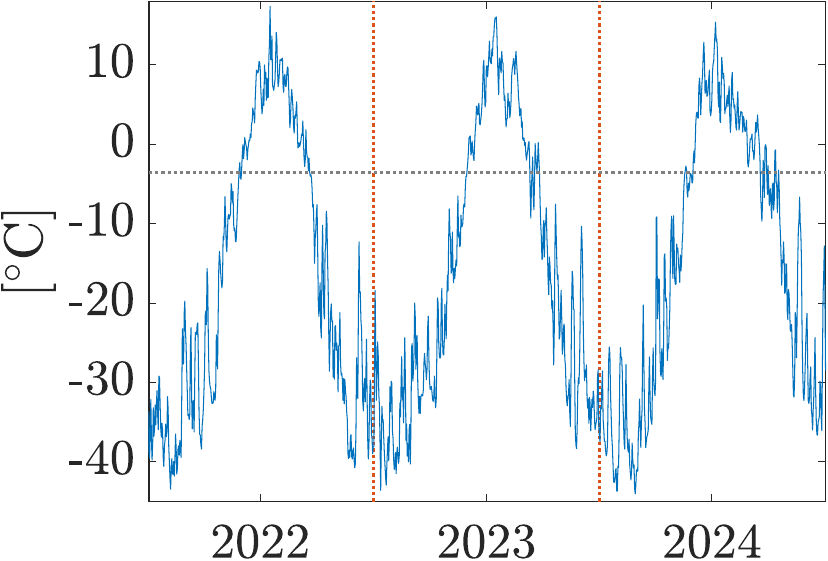}
    \end{minipage}
    \begin{minipage}[t]{0.28\linewidth}
        \centering
        \subfiguretitle{(b)}
        \vspace*{0.5ex}
        \includegraphics[height=3.5cm]{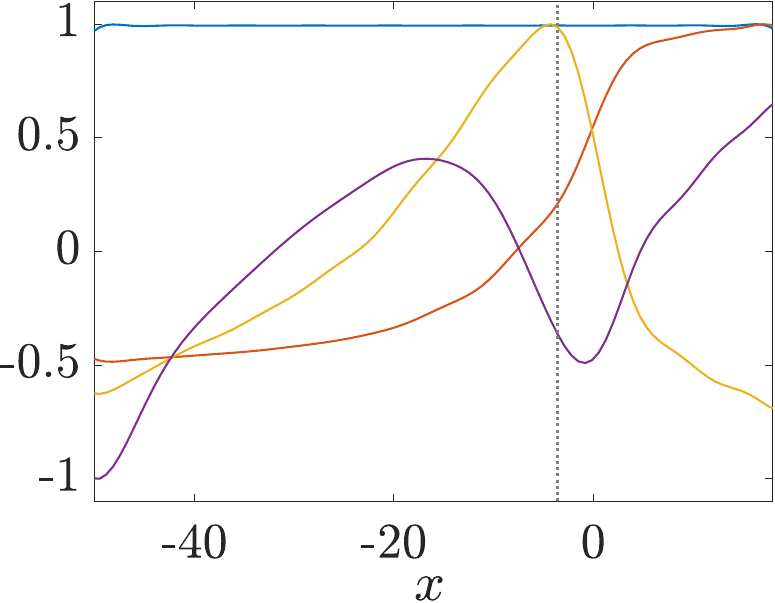}
    \end{minipage}
    \begin{minipage}[t]{0.28\linewidth}
        \centering
        \subfiguretitle{(c)}
        \vspace*{0.5ex}
        \includegraphics[height=3.5cm]{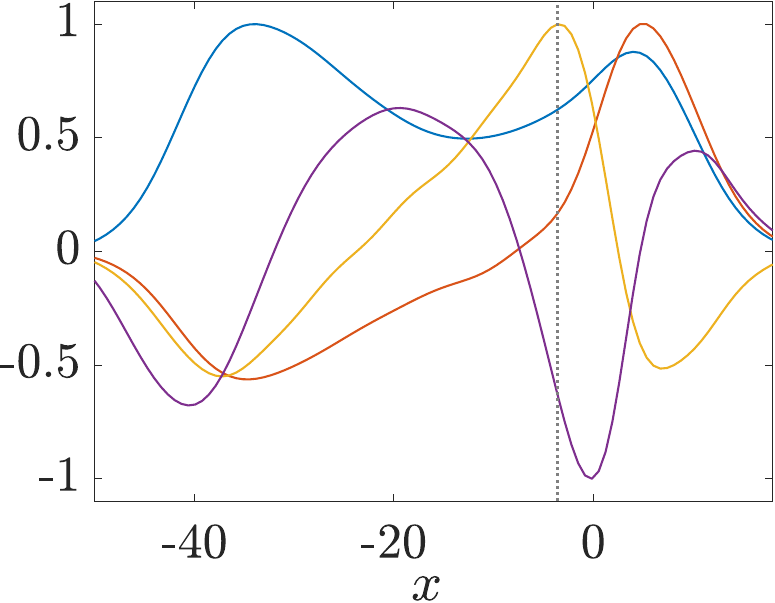}
    \end{minipage} \\[2ex]
    \begin{minipage}[t]{0.28\linewidth}
        \centering
        \subfiguretitle{(d)}
        \vspace*{0.75ex}
        \includegraphics[height=3.47cm]{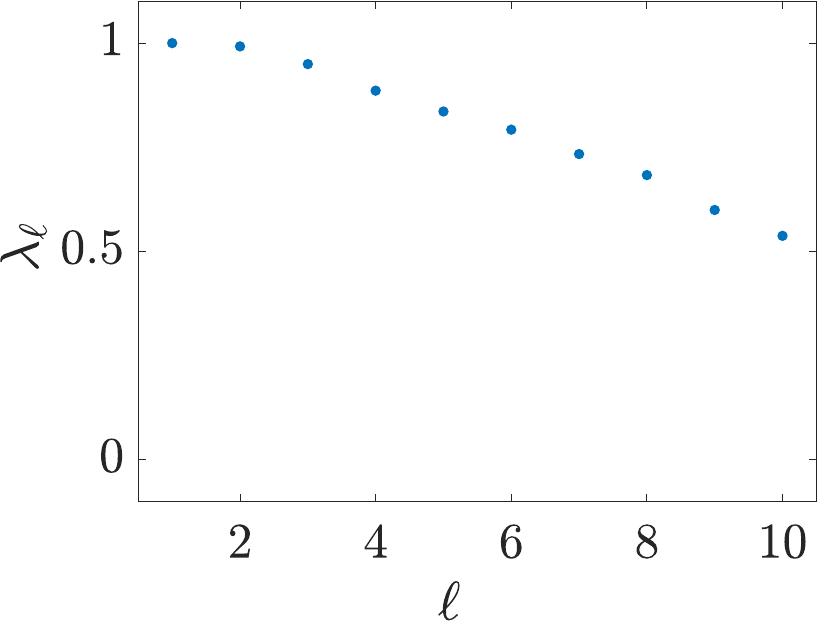}
    \end{minipage}
    \begin{minipage}[t]{0.28\linewidth}
        \centering
        \subfiguretitle{(e)}
        \includegraphics[height=3.5cm]{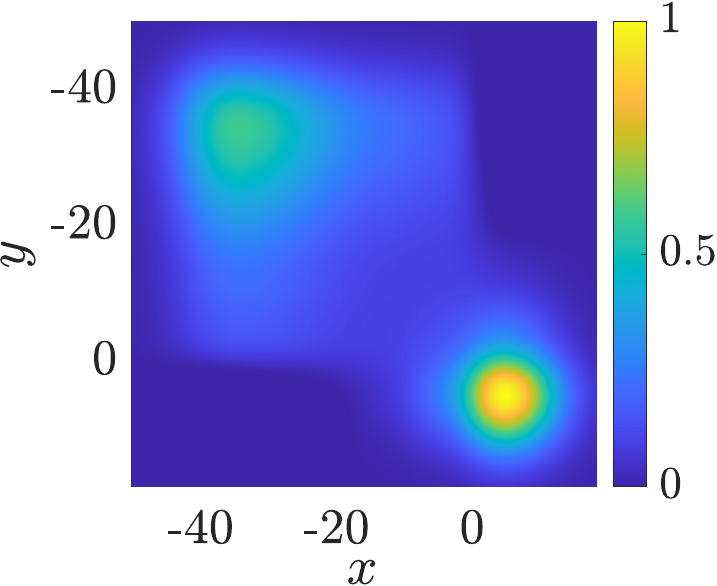}
    \end{minipage}
    \begin{minipage}[t]{0.28\linewidth}
        \centering
        \subfiguretitle{(f)}
        \includegraphics[height=3.5cm]{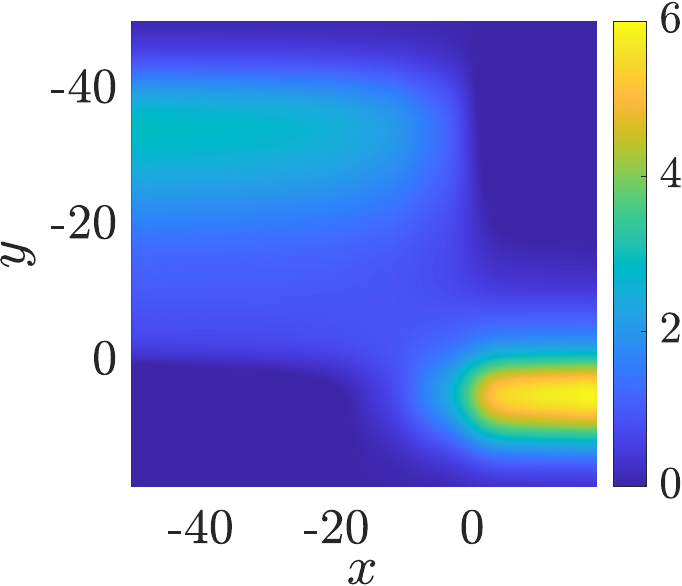}
    \end{minipage}
    \caption{(a)~Eureka's daily average temperature for the years 2022 to 2024. (b)~Dominant four eigenfunctions of the Koopman operator, where \cdash{matlab1} denotes the first, \cdash{matlab2} the second, \cdash{matlab3} the third, and \cdash{matlab4} the fourth eigenfunction. The dotted gray line represents the clustering into two sets. (c)~Dominant four eigenfunctions of the Perron--Frobenius operator. (d)~Eigenvalues of the operators. There is no clear spectral gap in this case. (e)~Rank-2 approximation of the estimated graphon $ w $. (f)~Rank-2 approximation of the estimated transition probability density $ p $.}
    \label{fig:temperature data}
\end{figure*}

In this demonstration we consider the daily average temperature in Eureka from 2005 to 2024, downloaded from \href{https://eureka.weatherstats.ca}{eureka.weatherstats.ca}. A snapshot of the data for the years 2022 to 2024 is shown in Figure~\ref{fig:temperature data}\ts(a), showing the expected seasonal variation between warm and cold temperatures over the course of each year. Our goal is to learn a graphon that approximately governs this data and leads to the transition probabilities between temperatures. Furthermore, we are again interested in the dominant eigenvalues and eigenfunctions associated with the data. For the sake of simplicity, we assume the system is reversible. This means that we seed the data-driven method with both the forward dataset $ \{(x^{(k)}, y^{(k)})\}_{k = 1}^m $ and the backward dataset $ \{(y^{(k)}, x^{(k)})\}_{k = 1}^m $ to learn a symmetric graphon. We begin by linearly scaling the temperature data to the interval $ [0, 1] $ to align with the theoretical presentation. We again employ a dictionary comprising 20 Gaussian functions with bandwidth $ \sigma = 0.05 $ and apply EDMD. The first four eigenfunctions of the Koopman and Perron--Frobenius operators, plotted as a function of temperature, are shown in Figures~\ref{fig:temperature data}\ts(b) and (c), respectively, and their corresponding eigenvalues in Figure~\ref{fig:temperature data}\ts(d). The lack of a clear spectral gap indicates that there are no well-defined clusters in this case. The resulting rank-2 reconstruction of the graphon $ w $ and the transition density function $ p $ are shown in Figures~\ref{fig:temperature data}\ts(e) and (f). The densities $ p(x, \,\cdot\,) $ predict, as expected, that the temperature tomorrow will be close to the temperature today. Note, however, that the variance of the temperature during the winter is larger than during the summer. Splitting the graphon into two clusters, the boundary is at approximately $ -4^\circ $C and we detect the two peaks around $ -33^\circ $C and $ 5^\circ $C. Using more eigenfunctions for the reconstruction leads to additional smaller peaks.

%%%%%%%%%%%%%%%%%%%%%%%%%%%%%%%%%%%%%%%%%%%%%%%%%%%%%%%%%%%%%%%%%%%%%%%%%%%%%%%%%%%%%%%%%%
\section{Extension to asymmetric graphons}
\label{sec:Extension to asymmetric graphons}

Compared to their symmetric counterparts, asymmetric graphons remain relatively unexplored in the literature and getting results akin to the symmetric case is much more delicate. In this section we will define the relevant transfer operators for asymmetric graphons and provide numerical results demonstrating how spectral clustering can still be applied to the asymmetric case. We will furthermore show that it is possible to learn the transition density function $ p $ or low-rank approximations thereof, but not the graphon $ w $ itself. However, a more rigorous analysis of the asymmetric case remains outstanding.

%%%%%%%%%%%%%%%%%%%%%%%%%%%%%%%%%%%%%%%%%%%%%%%%%%%%%%%%%%%%%%%%%%%%%%%%%%%%%%%%%%%%%%%%%%
\subsection{Transfer operators}

Given a (not necessarily symmetric) graphon $ w $, we again consider the transition probabilities according to Definition~\ref{def:transition density function}. However, we cannot rely on the existence of a uniquely defined invariant density anymore and need to consider different reweighted Hilbert spaces. Given a strictly positive reference density $ \mu $, let $ \nu = \mathcal{P} \mu $. We assume that $ \nu $ is strictly positive as well. Following the derivations in~\cite{KWNS18}, we can now define transfer operators for asymmetric graphons.

\begin{definition}[Koopman and reweighted Perron--Frobenius operators]
Given an asymmetric graphon $ w $ with associated transition density function $ p $, let $ u $ be a probability density with respect to $ \mu $ so that $ u = \frac{\rho}{\mu} $, where $ \rho $ is a probability density.
\begin{enumerate}[label=\roman*)]
\item The Koopman operator $ \mathcal{K} \colon L_\nu^2 \to L_\mu^2 $ is defined by
\begin{equation*}
    \mathcal{K} f(x) = \int p(x, y) \ts f(y) \ts \mathrm{d}y = \mathbb{E}[f(\Theta(x))].
\end{equation*}
\item The \emph{reweighted Perron--Frobenius operator} $ \mathcal{T} \colon L_\mu^2 \to L_\nu^2 $ is defined by
\begin{equation*}
    \mathcal{T} u(x) = \frac{1}{\nu(x)} \int p(y, x) \ts \mu(y) \ts u(y) \ts \mathrm{d}y.
\end{equation*}
\end{enumerate}
\end{definition}

The operator $ \mathcal{T} $ maps densities with respect to $ \mu $ to densities with respect to $ \nu $. For symmetric graphons, choosing $ \mu = \pi $ implies $ \nu = \pi $ so that we obtain the reweighted Perron--Frobenius operator introduced in Definition~\ref{def:reweighted Perron-Frobenius operator} as a special case. In order to be able to use the spectral theory developed for self-adjoint operators again, we define two additional operators. The relationships between all these operators will be detailed below.

\begin{definition}[Forward--backward and backward--forward operators]
Let $ u = \frac{\rho}{\mu} $ be a probability density with respect to $ \mu $ again and $ \rho $ a probability density.
\begin{enumerate}[label=\roman*)]
\item We define the \emph{forward--backward operator} $ \mathcal{F} \colon L_\mu^2 \to L_\mu^2 $ by
\begin{equation*}
    \mathcal{F} u(x) = \int p(x, y) \frac{1}{\nu(y)} \int p(z, y) \ts \mu(z) \ts u(z) \ts \mathrm{d}z \ts \mathrm{d}y.
\end{equation*}
\item Similarly, we define the \emph{backward--forward operator} $ \mathcal{B} \colon L_\nu^2 \to L_\nu^2 $ by
\begin{equation*}
    \mathcal{B} f(x) = \frac{1}{\nu(x)} \int p(y, x) \ts \mu(y) \ts \int p(y, z) \ts f(z) \ts \mathrm{d}z \ts \mathrm{d}y.
\end{equation*}
\end{enumerate}
\end{definition}

Note that $ \mathcal{F} = \mathcal{K} \mathcal{T} $ and $ \mathcal{B} = \mathcal{T} \mathcal{K} $. In related work, these operators have be used to detect coherent sets \cite{Froyland13, BaKo17, KHMN19}.

\begin{proposition}
The operators satisfy:
\begin{enumerate}[label=\roman*)]
\item $ \innerprod{\mathcal{T} u}{f}_{\nu} = \innerprod{u}{\mathcal{K} f}_{\mu} $.
\item $ \mathcal{T} \mathds{1} = \mathds{1} $ and $ \mathcal{K} \mathds{1} = \mathds{1} $.
\item $ \innerprod{\mathcal{F} u}{v}_\mu = \innerprod{u}{\mathcal{F} v}_\mu $ and $ \innerprod{\mathcal{B} f}{g}_\nu = \innerprod{f}{\mathcal{B} g}_\nu $.
\item The spectra of $ \mathcal{F} $ and $ \mathcal{B} $ are contained in the unit disk.
\end{enumerate}
\end{proposition}

\begin{proof}
We omit detailed proofs here since these results can be adapted from the stochastic differential equation case found in \cite{Froyland13, BaKo17} and are similar to the proofs of the results in Proposition~\ref{pro:properties of transfer operators}. We just briefly illustrate why $ \mathcal{F} $ is indeed a positive operator.
\begin{enumerate}[label=\roman*)] \setcounter{enumi}{2}
\item It holds that
\begin{align*}
    \innerprod{\mathcal{F} u}{v}_\mu &= \innerprod{\mathcal{K} \mathcal{T} u}{v}_\mu = \innerprod{\mathcal{T} u}{\mathcal{T} v}_\nu \\
       &= \innerprod{u}{\mathcal{K} \mathcal{T} v}_\mu = \innerprod{u}{\mathcal{F} v}_\mu
\end{align*}
and $ \innerprod{\mathcal{F} u}{u}_\mu = \norm{\mathcal{T} u}_\nu^2 \ge 0 $.  \qedhere
\end{enumerate}
\end{proof}

%%%%%%%%%%%%%%%%%%%%%%%%%%%%%%%%%%%%%%%%%%%%%%%%%%%%%%%%%%%%%%%%%%%%%%%%%%%%%%%%%%%%%%%%%%
\subsection{Spectral decompositions}

Since the operator $ \mathcal{F} $ is self-adjoint, it can be decomposed as $\mathcal{F} = \sum_\ell \lambda_\ell (\varphi_\ell \otimes_\mu \varphi_\ell)$,
where $ \varphi_\ell $ is now the $ \ell $th eigenfunction of the forward--backward operator. %Using Lemma~\ref{lem:SVD},
We then obtain the singular value decomposition $\mathcal{T} = \sum_\ell \sigma_\ell (u_\ell \otimes_\mu v_\ell),$
where $ \sigma_\ell = \lambda_\ell^{\frac{1}{2}} $, $ u_\ell = \lambda_\ell^{-\frac{1}{2}} \ts \mathcal{T} \varphi_\ell $, and $ v_\ell = \varphi_\ell $. This implies that the transition density function can be written as
\begin{equation*}
    p(x, y) = \sum_\ell \sigma_\ell \ts v_\ell(x) \ts u_\ell(y) \ts \nu(y).
\end{equation*}
As shown in Section~\ref{sec:Transfer operators of symmetric graphons}, it is possible to reconstruct symmetric graphons up to a multiplicative factor using the eigenfunctions of the Perron--Frobenius and Koopman operators. Let $ w $ now be an asymmetric graphon and define $ \widetilde{w}(x, y) = c(x) \ts w(x, y) $
where $ c $ is an arbitrary positive function.
Then $ \widetilde{d}_\text{out}(x) = c(x) \ts d_\text{out}(x) $ and
\begin{equation*}
    \widetilde{p}(x, y) = \frac{c(x) \ts w(x, y)}{c(x) \ts d_\text{out}(x)} = p(x, y).
\end{equation*}
We can thus in this case not expect to be able to recover the graphon $ w $ from the transition probability density $ p $ anymore.

%%%%%%%%%%%%%%%%%%%%%%%%%%%%%%%%%%%%%%%%%%%%%%%%%%%%%%%%%%%%%%%%%%%%%%%%%%%%%%%%%%%%%%%%%%
\subsection{Learning transfer operators from data}

Given random walk data $ \{ x^{(1)}, x^{(2)}, \dots, x^{(m+1)} \} $, we again define $ \big\{ (x^{(k)}, y^{(k)}) \big\}_{k=1}^m $, where $ y^{(k)} = x^{(k+1)} $. Assume now that $ x^{(k)} \sim \mu $, which implies $ y^{(k)} \sim \nu $. Computing the uncentered covariance and cross-covariance matrices $ \widetilde{C}_{xx} $ and $ \widetilde{C}_{xy} $ yields
\begin{alignat*}{2}
    \widetilde{C}_{xx} &:= \frac{1}{m} \sum_{k=1}^m \phi(x^{(k)}) \otimes \phi(x^{(k)}) \\ &\underset{\scriptscriptstyle m \rightarrow \infty}{\longrightarrow} \int_0^1 \phi(x) \otimes \phi(x) \ts \mu(x) \ts \mathrm{d}x = C_{xx}, \\
    \widetilde{C}_{xy} &:= \frac{1}{m} \sum_{k=1}^m \phi(x^{(k)}) \otimes \phi(y^{(k)}) \\ &\underset{\scriptscriptstyle m \rightarrow \infty}{\longrightarrow}
    \int_0^1 \phi(x) \otimes \mathcal{K} \phi(x) \ts \mu(x) \ts \mathrm{d}x = C_{xy}.
\end{alignat*}
The only difference is that the matrix $ \widetilde{K} = \widetilde{C}_{xx}^{-1} \ts \widetilde{C}_{xy} $ is now a data-driven Galerkin approximation of the Koopman operator with respect to the $ \mu $-weighted inner product, see Appendix~\ref{app:Galerkin approximation}.
The eigenvalues of the matrix $ \widetilde{K} $ are in general complex-valued. Instead of using the eigenfunctions of the Koopman operator for spectral clustering, the works \cite{Froyland13, BaKo17} suggested using the forward--backward operator $ \mathcal{F} $ instead, which corresponds to detecting finite-time coherent sets. The Galerkin projection of the operator $ \mathcal{T} $ can be approximated by $ \widetilde{T} = \widetilde{C}_{yy}^{-1} \ts \widetilde{C}_{yx} $ (see Lemma~\ref{lem:Galerkin adjoint}) and we define $ \widetilde{F} = \widetilde{C}_{xx}^{-1} \ts \widetilde{C}_{xy} \ts \widetilde{C}_{yy}^{-1} \ts \widetilde{C}_{yx} $, which is a composition of two Galerkin approximations \cite{KD24}. This leads to the following spectral clustering algorithm for asymmetric graphons.

\begin{talgorithm}[Data-driven spectral clustering algorithm for asymmetric graphons] $ $
\begin{enumerate} \setlength{\itemsep}{0ex}
\item Given training data $ \{ (x^{(k)}, y^{(k)}) \}_{k=1}^m $ and the vector-valued observable $ \phi $.
\item Compute the $ r $ dominant eigenvalues $ \lambda_\ell $ and associated eigenvectors $ \xi_\ell $ of $ \widetilde{F} $.
\item Evaluate the resulting eigenfunctions $ \widetilde{\varphi}_\ell(x) := \xi_\ell^\top \phi(x) $ in all points $ x^{(k)} $.
\item Let $ s_i \in \R^{r} $ denote the (transposed) $ i $th row of
\begin{equation*}
    \boldsymbol{\widetilde{\varphi}} =
    \begin{bmatrix}
        \widetilde{\varphi}_1(x^{(1)}) & \dots & \widetilde{\varphi}_r(x^{(1)}) \\
        \vdots & \ddots & \vdots \\
        \widetilde{\varphi}_1(x^{(m)}) & \dots & \widetilde{\varphi}_r(x^{(m)}) \\
    \end{bmatrix}
    \in \R^{m \times r}.
\end{equation*}
\item Cluster the points $ \{ \ts s_i \ts \}_{i=1}^m $ into $ r $ clusters using, e.g., $ k $-means.
\end{enumerate}
\end{talgorithm}

Note that the eigenfunctions $ \varphi_\ell $ of the forward--backward operator $ \mathcal{F} $ are the right singular functions $ v_\ell $ of the operator $ \mathcal{T} $. The corresponding left singular functions $ u_\ell $, which are required for reconstructing the transition probability densities $ p $, can then be approximated via
\begin{equation*}
    \begin{split}
        &v_\ell(x) = \varphi_\ell(x) \approx \xi_\ell^\top \phi(x)\\ &\implies u_\ell(x) = \lambda_\ell^{-\frac{1}{2}} \ts \mathcal{T} \varphi_\ell(x) \approx \lambda_\ell^{-\frac{1}{2}} \big(\widetilde{C}_{yy}^{-1} \ts \widetilde{C}_{yx} \ts \xi_\ell\big)^\top \phi(x).
    \end{split}
\end{equation*}
Moreover, the density $ \nu $ can be estimated from the $ y^{(k)} $ data using kernel density estimation.

%%%%%%%%%%%%%%%%%%%%%%%%%%%%%%%%%%%%%%%%%%%%%%%%%%%%%%%%%%%%%%%%%%%%%%%%%%%%%%%%%%%%%%%%%%
\subsection{Numerical demonstrations}

We will now illustrate the proposed data-driven clustering approach, show how the transition probability densities can be estimated, and highlight the main differences between symmetric and asymmetric graphons.

%%%%%%%%%%%%%%%%%%%%%%%%%%%%%%%%%%%%%%%%%%%%%%%%%%%%%%%%%%%%%%%%%%%%%%%%%%%%%%%%%%%%%%%%%%
\subsubsection{Synthetic data}

\begin{figure*}
    \definecolor{matlab1}{RGB}{0, 114, 189}
    \definecolor{matlab2}{RGB}{217, 83, 25}
    \definecolor{matlab3}{RGB}{237, 177, 32}
    \definecolor{matlab4}{RGB}{126, 47, 142}
    \newcommand{\cdash}[1]{\textcolor{#1}{\rule[0.5ex]{1em}{0.3ex}}}
    \centering
    \begin{minipage}[t]{0.28\linewidth}
        \centering
        \subfiguretitle{(a)}
        \vspace*{0.5ex}
        \includegraphics[height=3.5cm]{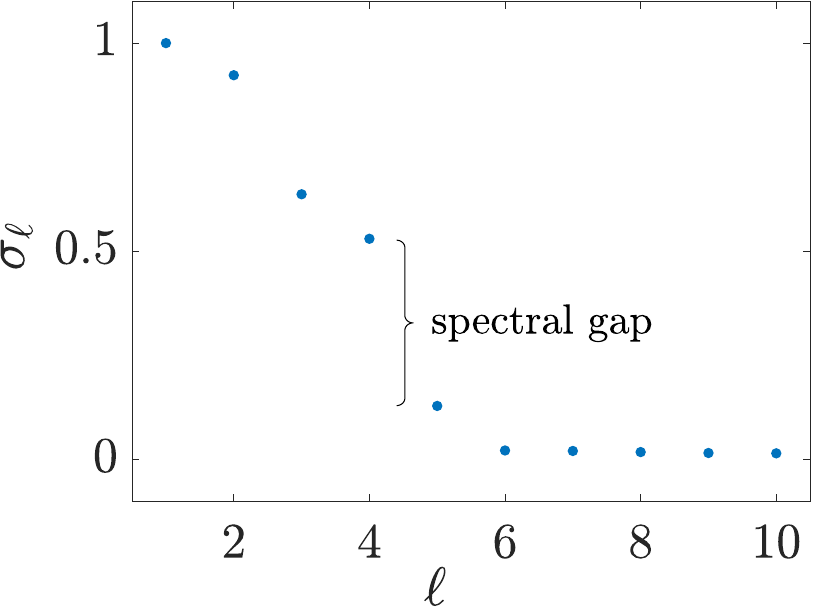}
    \end{minipage}
    \begin{minipage}[t]{0.28\linewidth}
        \centering
        \subfiguretitle{(b)}
        \vspace*{0.5ex}
        \includegraphics[height=3.5cm]{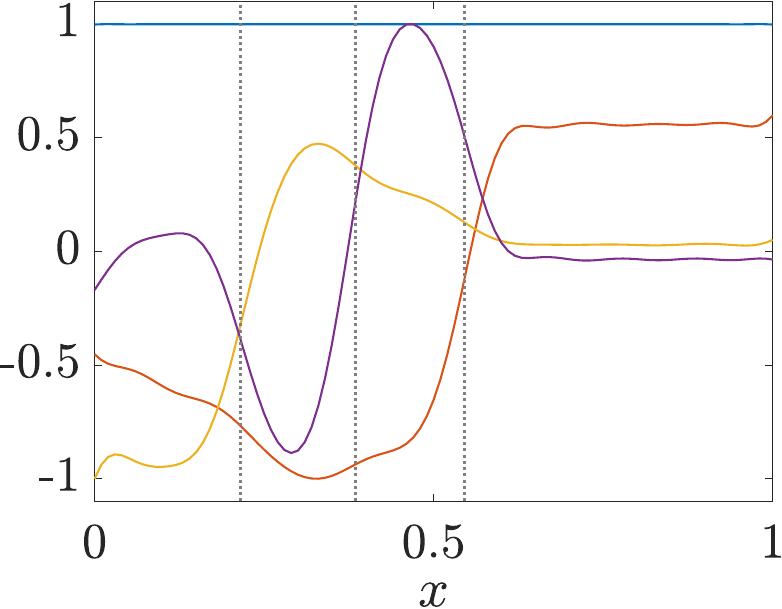}
    \end{minipage}
    \begin{minipage}[t]{0.28\linewidth}
        \centering
        \subfiguretitle{(c)}
        \vspace*{0.5ex}
        \includegraphics[height=3.5cm]{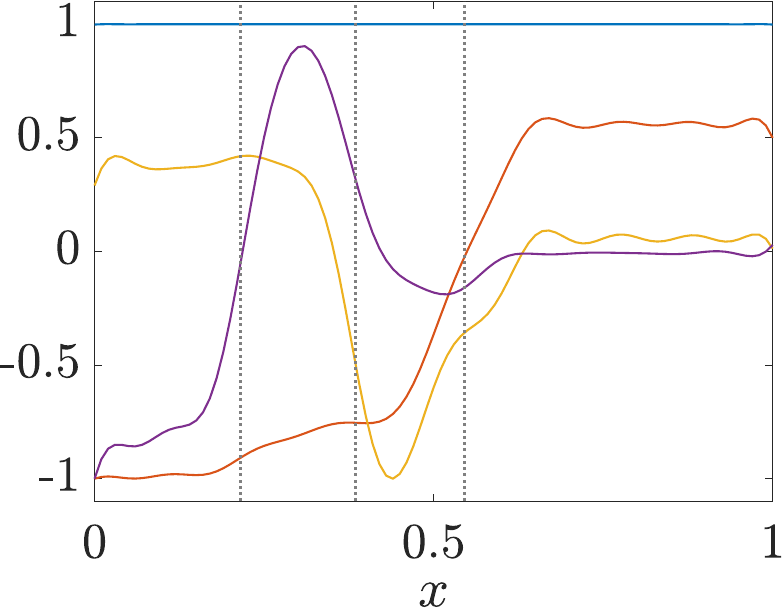}
    \end{minipage} \\[2ex]
    \begin{minipage}[t]{0.28\linewidth}
        \centering
        \subfiguretitle{(d)}
        \includegraphics[height=3.5cm]{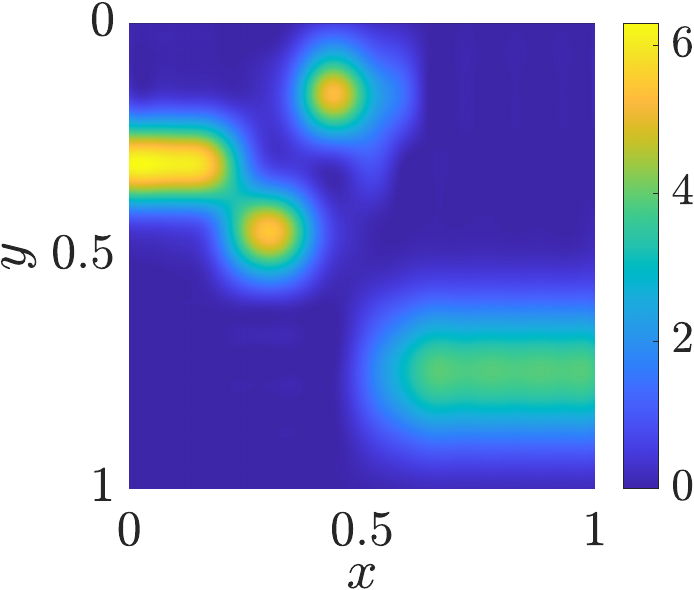}
    \end{minipage}
    \begin{minipage}[t]{0.28\linewidth}
        \centering
        \subfiguretitle{(e)}
        \includegraphics[height=3.5cm]{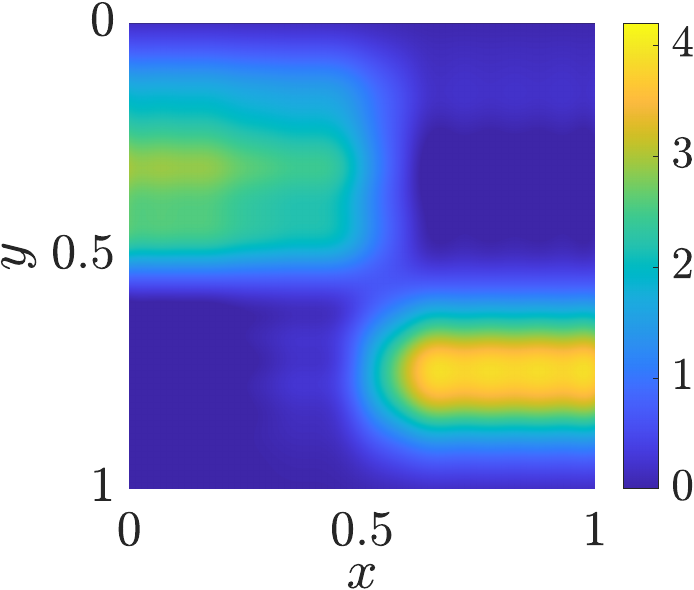}
    \end{minipage}
    \begin{minipage}[t]{0.28\linewidth}
        \centering
        \subfiguretitle{(f)}
        \includegraphics[height=3.5cm]{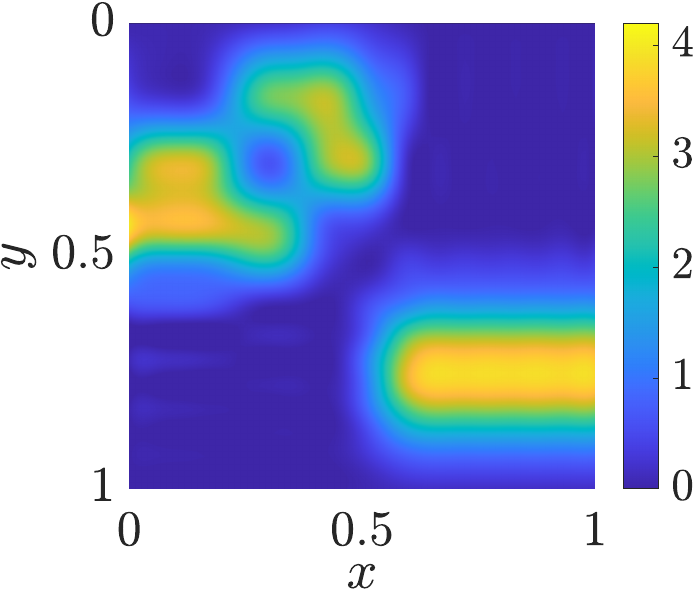}
    \end{minipage}
    \caption{(a)~Numerically computed singular values. (b)~First four left singular functions of the reweighted Perron--Frobenius operator $ \mathcal{T} $ using  20 Gaussians, where \cdash{matlab1} denotes the first, \cdash{matlab2} the second, \cdash{matlab3} the third, and \cdash{matlab4} the fourth singular function. The dotted gray lines separate the detected clusters. (c)~Corresponding right singular functions. (d)~Resulting rank-4 approximation of $ p $ using the singular functions. (e)~Rank-2 approximation using the singular functions. (f)~Rank-4 approximation of $ p $ using the eigenfunctions of the Koopman operator assuming reversibility.}
    \label{fig:quadruple-peak eigenfunctions}
\end{figure*}

As a first example, we generate a random walk of length $m = 20\ts000$ with transition probabilities derived from the asymmetric graphon in Example~\ref{ex:guiding examples}. This graphon exhibits a cluster in the interval $ [0.5, 1] $, while also having a cycle through three clusters in the interval $ [0, 0.5] $, as illustrated in Figure~\ref{fig:guiding examples}(f), within which the cycling between the clusters in the interval $ [0, 0.5 ]$ is difficult to observe directly. Our numerical demonstration that follows implements EDMD using 20 evenly-spaced Gaussian functions, each with bandwidth $ \sigma = 0.05 $. A clear spectral gap is shown in the singular values presented in Figure~\ref{fig:quadruple-peak eigenfunctions}\ts(a). The leading left and right singular functions of the learned reweighted Perron--Frobenius operator $ \mathcal{T} $ are presented in Figure~\ref{fig:quadruple-peak eigenfunctions}\ts(b) and (c), respectively, while a rank-4 approximation of the transition probability function $ p $ is given in panel (d); compare with Figure~\ref{fig:guiding examples}(e). Importantly, if the stochastic process is assumed to be reversible, then the cycling behavior of the random walk in $ [0, 0.5] $ is reduced to a single cluster. That is, an assumed symmetric process results in only two clusters and completely fails to identify the nuances of the cycling behavior between the three clusters in $ [0, 0.5] $. The transition probabilities $ C $ between the four detected clusters of the asymmetric graphon, estimated again from the training data, are
\begin{equation*}
    C =
    \begin{bmatrix}%[rrrr]
        7.8 & \mathbf{76.1} & 15.4 & 0.7 \\
        12.8 & 24.2 & \mathbf{58.1} & 4.9 \\
        \mathbf{60.8} & 16.1 & 8.6 & 14.5 \\
        2.1 & 0.4 & 1.9 & \mathbf{95.6}
    \end{bmatrix},
\end{equation*}
where $ c_{ij} $ is the probability (in per cent) of going from cluster $ i $ to cluster $ j $. The cycling behavior is seen in this transition matrix with high probabilities of transitioning from cluster 1 to 2, 2 to 3, and 3 to 1.

%%%%%%%%%%%%%%%%%%%%%%%%%%%%%%%%%%%%%%%%%%%%%%%%%%%%%%%%%%%%%%%%%%%%%%%%%%%%%%%%%%%%%%%%%%
\subsubsection{Non-reversible stochastic process}

\begin{figure*}
    \definecolor{matlab1}{RGB}{0, 114, 189}
    \definecolor{matlab2}{RGB}{217, 83, 25}
    \definecolor{matlab3}{RGB}{237, 177, 32}
    \definecolor{matlab4}{RGB}{126, 47, 142}
    \definecolor{matlab5}{RGB}{119, 172, 48}
    \newcommand{\cdash}[1]{\textcolor{#1}{\rule[0.5ex]{1em}{0.3ex}}}
    \centering
    \begin{minipage}[t]{0.28\linewidth}
        \centering
        \subfiguretitle{(a)}
        \vspace*{0.8ex}
        \includegraphics[height=3.5cm]{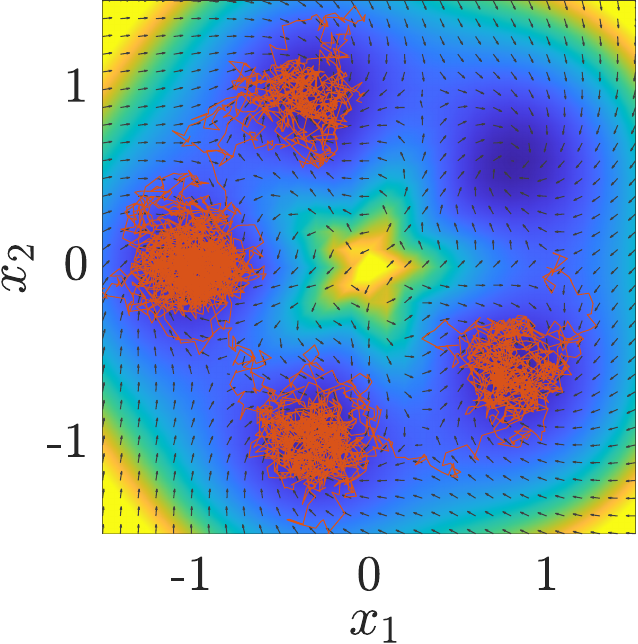}
    \end{minipage}
    \begin{minipage}[t]{0.28\linewidth}
        \centering
        \subfiguretitle{(b)}
        \includegraphics[height=3.65cm]{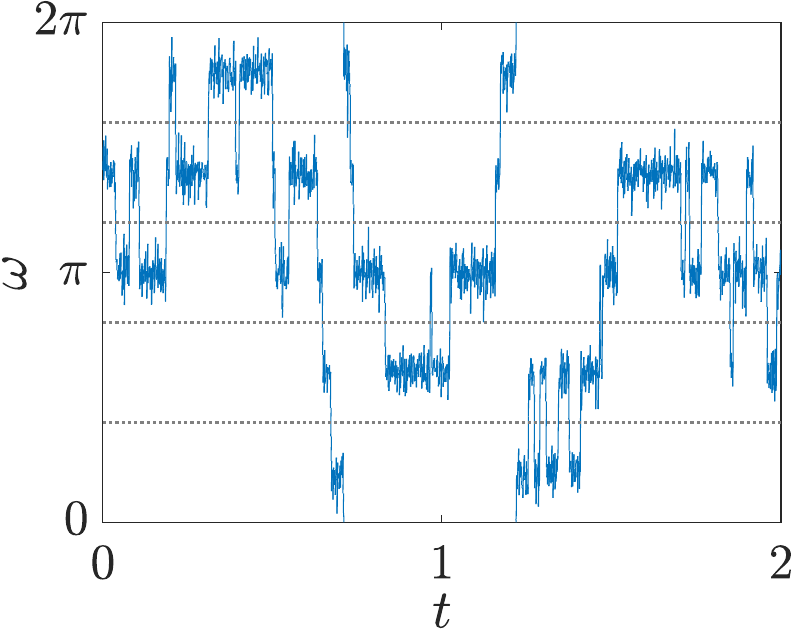}
    \end{minipage}
    \begin{minipage}[t]{0.30\linewidth}
        \centering
        \subfiguretitle{(c)}
        \vspace*{0.75ex}
        \includegraphics[height=3.5cm]{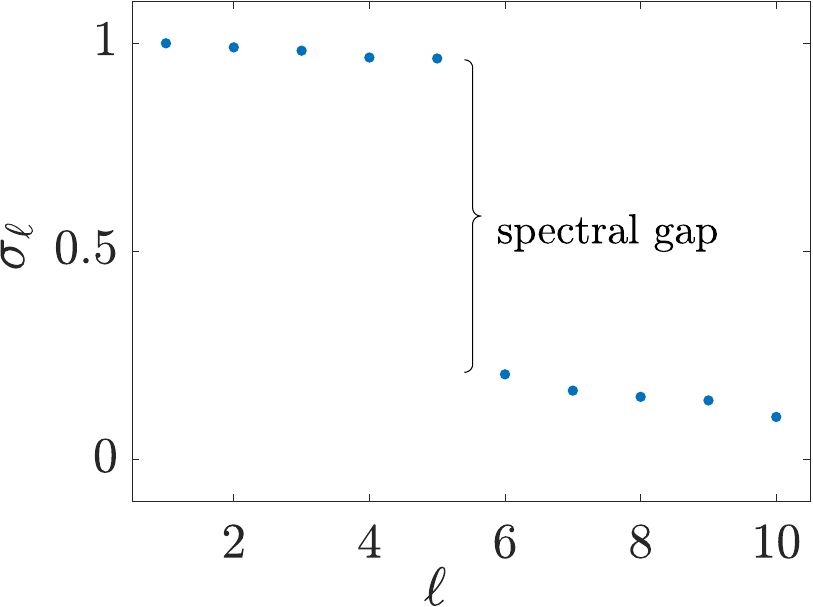}
    \end{minipage} \\[2ex]
    \begin{minipage}[t]{0.28\linewidth}
        \centering
        \subfiguretitle{(d)}
        \vspace*{0.8ex}
        \includegraphics[height=3.5cm]{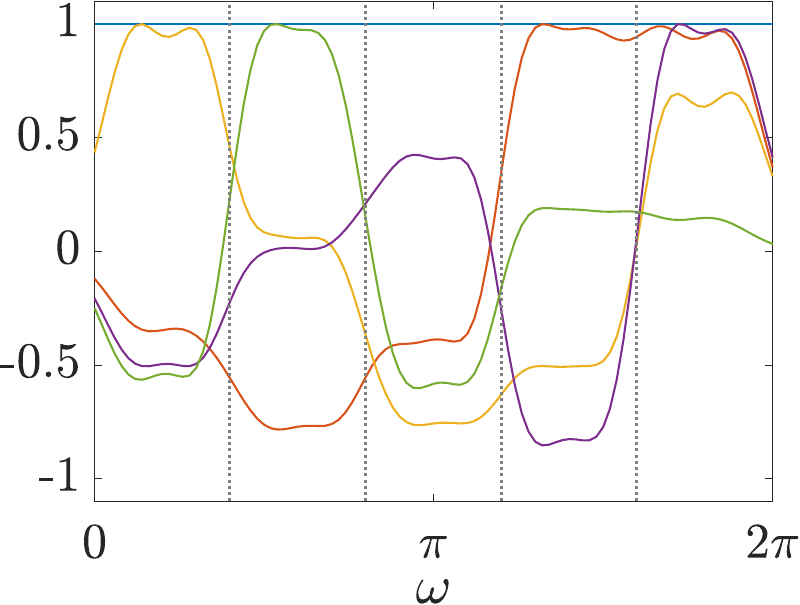}
    \end{minipage}
    \begin{minipage}[t]{0.28\linewidth}
        \centering
        \subfiguretitle{(e)}
        \vspace*{0.8ex}
        \includegraphics[height=3.5cm]{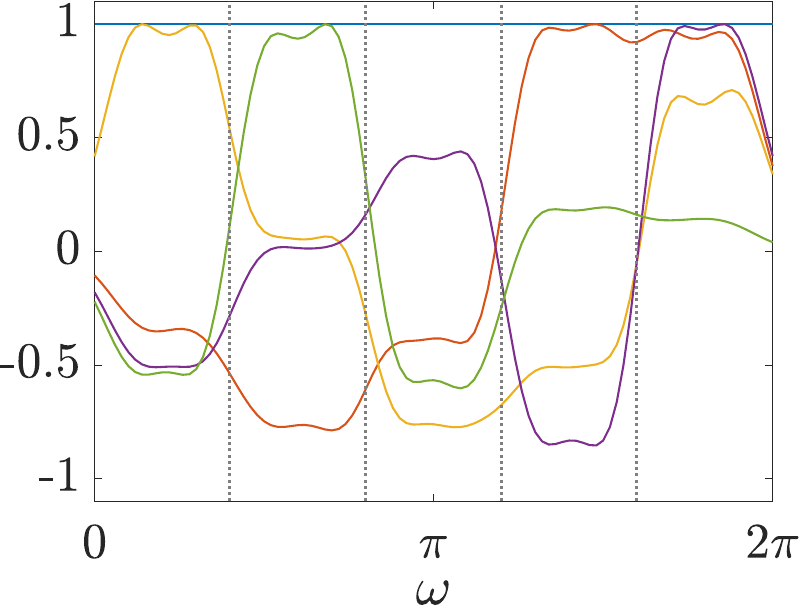}
    \end{minipage}
    \begin{minipage}[t]{0.30\linewidth}
        \centering
        \subfiguretitle{(f)}
        \includegraphics[height=3.65cm]{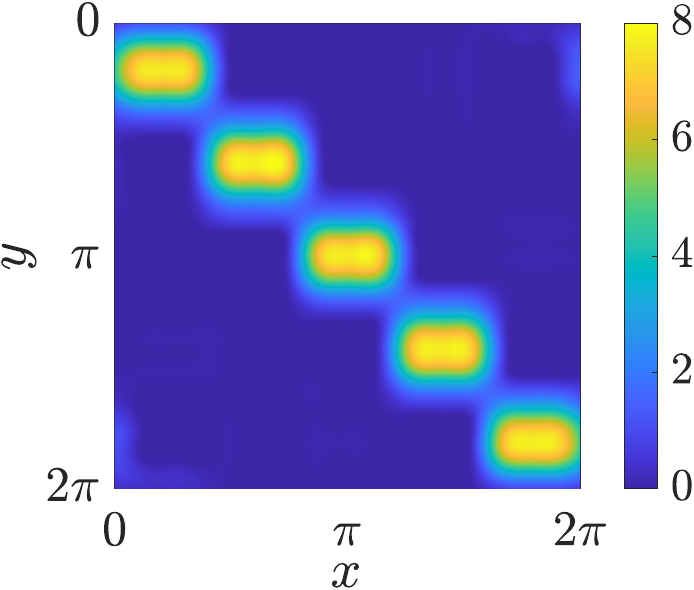}
    \end{minipage}
    \caption{(a)~Lemon-slice potential comprising five wells and one trajectory evolving within it according to the non-reversible system \eqref{eq:SDE}. (b)~Angular coordinate $ \omega $ of the training data and the resulting clustering into five sets. (c)~First ten singular values, illustrating that there is a clear spectral gap between the fifth and sixth singular value. (d)~Dominant five left singular functions, where \cdash{matlab1} denotes the first, \cdash{matlab2} the second, \cdash{matlab3} the third, \cdash{matlab4} the fourth, and \cdash{matlab5} the fifth singular function. (e)~Dominant five right singular functions. (f)~Resulting rank-5 reconstruction of the transition probability density. The small light-blue regions in the top right and bottom left corners correspond to transitions between adjacent clusters in the $ 2 \ts \pi $-periodic domain.}
    \label{fig:SDE data}
\end{figure*}

We now consider random walk data obtained from the non-reversible overdamped Langevin equation (see, e.g., \cite{LNP13}), defined by
\begin{equation} \label{eq:SDE}
    \mathrm{d}X_t = \big(\!-\!\nabla V(X_t) + M \ts \nabla V(X_t)\big) \ts \mathrm{d}t + \sqrt{2 \beta^{-1}} \ts \mathrm{d}W_t,
\end{equation}
where $ V \colon \mathbb{R}^2 \to \mathbb{R} $ is a potential, $ M \in \R^{2 \times 2} $ an antisymmetric matrix, and $ \beta > 0 $ the inverse temperature. We choose the two-dimensional potential
\begin{equation*}
     V(x) = \cos\big(k \, \arctan(x_2, x_1)\big) + 10 \left(\sqrt{x_1^2 + x_2^2} - 1\right)^2,
\end{equation*}
taken from \cite{BKKBDS18}, which comprises $ k $ wells that are uniformly distributed on the unit circle. In what follows, we set $ k = 5 $, $ M = \left[ \begin{smallmatrix} 0 & 1 \\ -1 & 0 \end{smallmatrix}\right] $, and $ \beta = 2 $. The potential, drift term, and resulting dynamics are illustrated in Figure~\ref{fig:SDE data}\ts(a). Since the slow dynamics, i.e., the transitions between the different wells of the system, predominantly depend on the angular coordinate $ \omega $ and not the radial coordinate $ r $, our data is taken only to be this angular coordinate in time. The angular trajectory data we use to estimate singular functions and the transition probabilities is shown in Figure~\ref{fig:SDE data}\ts(b), where the lag time is $ \tau = 0.1 $. The numerically computed singular values and the corresponding left and right singular functions are shown in Figures \ref{fig:SDE data}\ts(c), (d), and (e), respectively. There is a clear spectral gap between the fifth and sixth singular value, indicating the existence of five clusters corresponding to the five wells of the potential. Figure~\ref{fig:SDE data}\ts(f) shows a rank-5 reconstruction of the transition densities, illustrating how particles move from one well to the neighboring wells.

%%%%%%%%%%%%%%%%%%%%%%%%%%%%%%%%%%%%%%%%%%%%%%%%%%%%%%%%%%%%%%%%%%%%%%%%%%%%%%%%%%%%%%%%%%
\subsubsection{Nikkei index data}

As a final demonstration, we choose the Nikkei 225 index over the last five years, shown in Figure~\ref{fig:stock market data}\ts(a). After applying our data-driven method to this data we obtain the singular value spectrum presented in Figure~\ref{fig:stock market data}\ts(b), where a small spectral gap can be observed between the fourth and fifth singular values. Nevertheless, the clustering algorithm identifies meaningful boundaries, as can be observed in the rank-4 approximation of the transition densities in Figure~\ref{fig:stock market data}\ts(c). This approximation provides four distinct clusters that roughly correspond to the plateaus in the original data. The estimated transition density function now allows us to generate new random walk data, e.g., to predict the behavior in the future. The higher the rank of the approximation, the more accurate (in theory) the transition densities. However, the matrix approximations of the projected transfer operators may be ill-conditioned, which might lead to spurious eigenvalues and oscillatory eigenfunctions.

\begin{figure*}
    \centering
    \begin{minipage}[t]{0.30\linewidth}
        \centering
        \subfiguretitle{(a)}
        \includegraphics[height=3.5cm]{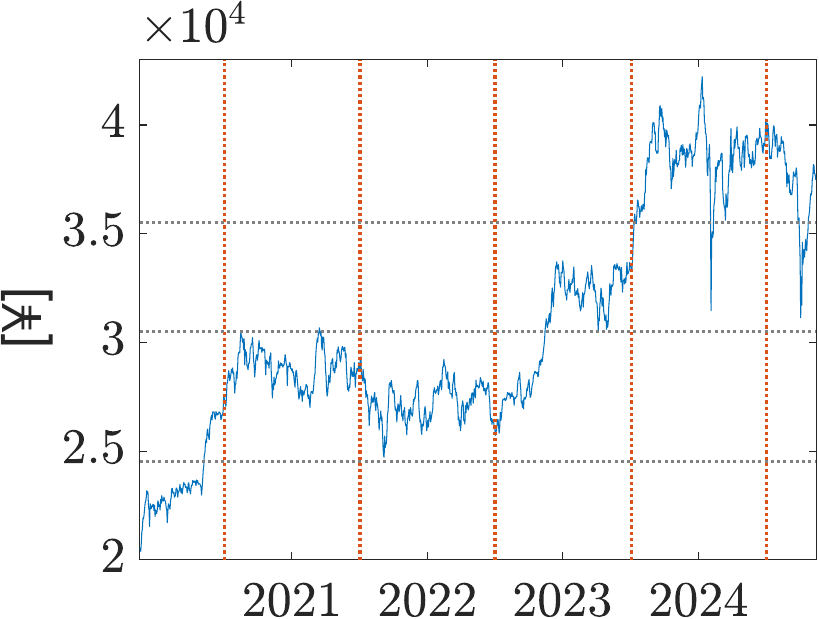}
    \end{minipage}
    \begin{minipage}[t]{0.30\linewidth}
        \centering
        \subfiguretitle{(b)}
        \vspace*{2.1ex}
        \includegraphics[height=3.44cm]{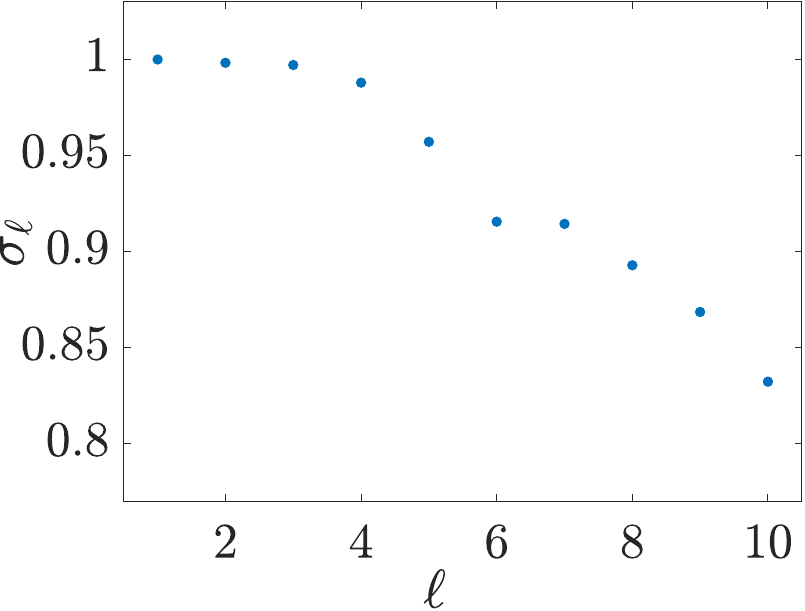}
    \end{minipage}
    \begin{minipage}[t]{0.28\linewidth}
        \centering
        \subfiguretitle{(c)}
        \vspace*{-0.1ex}
        \includegraphics[height=3.87cm]{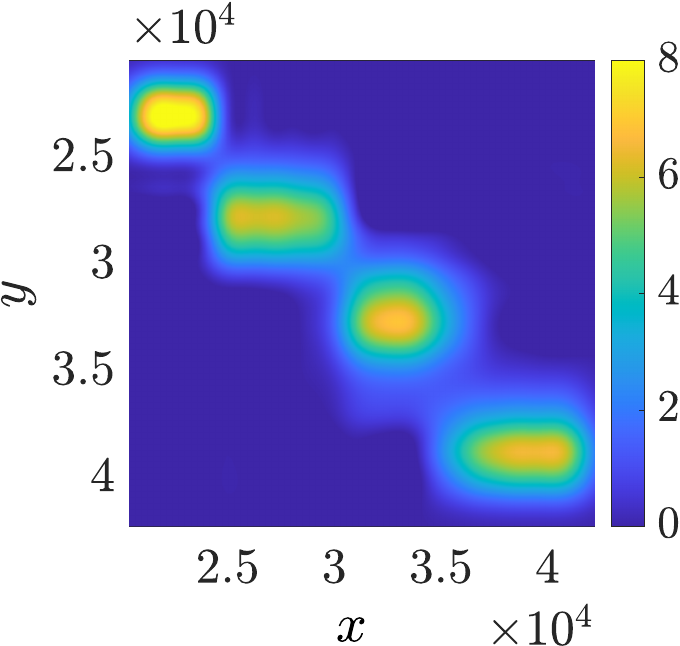}
    \end{minipage}
    \caption{(a)~Nikkei 225 index and the resulting clustering into four sets. (b)~Singular values. (c)~Rank-4 approximation of the transition density function.}
    \label{fig:stock market data}
\end{figure*}

%%%%%%%%%%%%%%%%%%%%%%%%%%%%%%%%%%%%%%%%%%%%%%%%%%%%%%%%%%%%%%%%%%%%%%%%%%%%%%%%%%%%%%%%%%
\section{Discussion}
\label{sec:Discussion}

The main goal of this work was to provide a data-driven method for processing stochastic signals. We have defined transfer operators associated with random walks on graphons and extended conventional spectral clustering techniques for undirected and directed graphs to symmetric and asymmetric graphons. We have furthermore shown that spectral decompositions of transfer operators allow us not only to identify clusters, but also to learn the underlying transition probability densities and, provided the random process is reversible, the graphon itself. The transfer operators can be either estimated from random walk data, using data-driven Galerkin projections, or, if the graphon is known, by directly computing the required integrals. Assuming there exists a spectral gap, i.e., there are only a few eigenvalues close to one, representing the slow timescales, our method successfully identifies clusters in symmetric and asymmetric graphons. We have demonstrated the efficacy, accuracy, and versatility of the proposed framework using synthetic and real-world data, ranging from simple benchmark problems to daily average temperatures and stock market data.

So far, prior work has mostly focused on symmetric graphons. Random walk processes associated with symmetric graphons are reversible and the spectra of the corresponding transfer operators are real-valued. Conventional spectral clustering methods that have been developed for undirected graphs \cite{Luxburg07} can be easily extended to symmetric graphons. The asymmetric graphon case, on the other hand, is not yet well understood. A random walk process might get trapped in absorbing sets and the eigenvalues of transfer operators are, in general, complex-valued. Considering the forward--backward process, which again results in self-adjoint operators, is one way to circumvent this issue. Our spectral clustering approach for asymmetric graphons is a generalization of the method proposed in \cite{KT24}. A possible next step could be to consider time-evolving graphons. The clusters then also change over time, meaning that they might emerge, vanish, split, or merge with other clusters. Tracking these changes is a challenging problem in even simple contrived examples, let alone real-world datasets. Transfer operator-based clustering methods for time-evolving graphs have recently been proposed in \cite{TDK25}. Possible extensions of these methods to time-evolving graphons will be considered in future work.

\section*{Acknowledgments}

We thank David Lloyd for interesting discussions about graphons and their applications. JJB was partially supported by an NSERC Discovery Grant and the Fondes de Recherche du Qu\'ebec -- Nature et Technologies (FRQNT).

%%%%%%%%%%%%%%%%%%%%%%%%%%%%%%%%%%%%%%%%%%%%%%%%%%%%%%%%%%%%%%%%%%%%%%%%%%%%%%%%%%%%%%%%%%
\appendix

%%%%%%%%%%%%%%%%%%%%%%%%%%%%%%%%%%%%%%%%%%%%%%%%%%%%%%%%%%%%%%%%%%%%%%%%%%%%%%%%%%%%%%%%%%
\section{Spectral decompositions of compact operators}
\label{app:spectral decompositions}

For the sake of completeness, we will briefly review spectral properties of compact operators. Detailed derivations and proofs can be found in~\cite{Aubin00, MSKS20}. In what follows, $ H $, $ H_1 $, and $ H_2 $ denote Hilbert spaces.

\begin{definition}[Rank-one operator]
Given nonzero elements $ r \in H_1 $ and $ s \in H_2 $, we define the bounded linear \emph{rank-one operator} $ s \otimes r \colon H_1 \to H_2 $ by $ (s \otimes r) f = \innerprod{r}{f} s $.
\end{definition}

Note that we adopt the physicists' convention and define the inner product to be linear in the second argument.

\begin{theorem}
Let $ \mathcal{A} \colon H \to H $ be a self-adjoint compact linear operator, then there exists an \emph{eigendecomposition}
\begin{equation*}
    \mathcal{A} = \sum_\ell \lambda_\ell (\varphi_\ell \otimes \varphi_\ell),
\end{equation*}
where $ \{ \varphi_\ell \}_{\ell} $ forms an orthonormal system of eigenfunctions corresponding to the nonzero eigenvalues $ \{ \lambda_\ell \}_\ell \subseteq \R $. If the set of eigenvalues is not finite, then the sequence of eigenvalues converges to zero.
\end{theorem}

We assume the eigenvalues $ \lambda_\ell $ to be sorted in non-increasing order, i.e., $ \lambda_1 \ge \lambda_2 \ge \lambda_3 \ge \dots $. Similarly, we can also compute a singular value decomposition of a compact linear operator.

\begin{theorem}
Let $ \mathcal{A} \colon H_1 \to H_2 $ be a compact linear operator, then there exists a \emph{singular value decomposition}
\begin{equation*}
    \mathcal{A} = \sum_\ell \sigma_\ell (u_\ell \otimes v_\ell),
\end{equation*}
where the left singular functions $ \{ u_\ell \}_\ell \subset H_2 $ and right singular functions $ \{ v_\ell \}_\ell \subset H_1 $ form orthonormal systems associated with the nonzero singular values $ \{ \sigma_\ell \}_\ell \subseteq \R_{> 0} $. If the set of singular values is not finite, then the sequence of singular values converges to zero.
\end{theorem}

We also assume the singular values to be sorted in non-increasing order, i.e., $ \sigma_1 \ge \sigma_2 \ge \sigma_3 \ge \dots $. Eigenfunctions and singular functions are closely related as the following lemma shows.

\begin{lemma} \label{lem:SVD}
Given a compact linear operator $ \mathcal{A} \colon H_1 \to H_2 $, it holds that $ \mathcal{A}^* \mathcal{A} \colon H_1 \to H_1 $ is a positive operator.\!\footnote{An operator $ \mathcal{A} $ is called positive if $ \mathcal{A} = \mathcal{A}^* $ and $ \innerprod{\mathcal{A} f}{f} \ge 0 $ for all $ f $.} Let $ \{ \varphi_\ell \}_\ell $ be the orthonormal system of eigenfunctions and $ \{ \lambda_\ell \}_\ell \subseteq \R_{>0} $ the set of associated nonzero eigenvalues of $ \mathcal{A}^* \mathcal{A} $. Then the singular value decomposition of $ \mathcal{A} $ is defined by the singular values $ \sigma_\ell = \lambda_\ell^{\frac{1}{2}} $ and the singular functions $ u_\ell = \lambda_\ell^{-\frac{1}{2}} \mathcal{A} \varphi_\ell $ and $ v_\ell = \varphi_\ell $.
\end{lemma}

%%%%%%%%%%%%%%%%%%%%%%%%%%%%%%%%%%%%%%%%%%%%%%%%%%%%%%%%%%%%%%%%%%%%%%%%%%%%%%%%%%%%%%%%%%
\section{Galerkin approximation of the Koopman operator}
\label{app:Galerkin approximation}

Here we provide more details on the Galerkin approximation of the Koopman operator $ \mathcal{K} \colon H_1 \to H_2 $ and its adjoint. In the symmetric graphon case we will have $ H_1 = H_2 = L_\pi^2 $, while in the asymmetric graphon case we will have $ H_1 = L_\mu^2 $ and $ H_2 = L_\nu^2 $. Given a set of $ n $ linearly independent basis functions $ \{ \phi_i \}_{i=1}^n \subset H_1 \cap H_2 $ spanning the $ n $-dimensional subspace $ \mathbb{V} = \mspan\{ \phi_i \}_{i=1}^n $, we define the vector-valued function
\begin{equation*}
    \phi(x) = [\phi_1(x), \dots, \phi_n(x)]^\top \in \R^n.
\end{equation*}
Any function $ f \in \mathbb{V} $ can hence be written as
\begin{equation*}
    f(x) = \sum_{i=1}^n \alpha_i \ts \phi_i(x) = \alpha^\top \phi(x),
\end{equation*}
with $ \alpha = [\alpha_1, \dots, \alpha_n]^\top \in \R^n $. In order to compute the Galerkin approximation $ \mathcal{K}_\phi \colon \mathbb{V} \to \mathbb{V} $ of the Koopman operator $ \mathcal{K} $ with respect to the inner product $ \innerprod{\cdot}{\cdot}_{H_1} $, we have to construct the two matrices $ C_{xx}, C_{xy} \in \R^{n \times n} $, with
\begin{equation*}
    \big[C_{xx}\big]_{ij} = \innerprod{\phi_i}{\phi_j}_{H_1}
    \quad \text{and} \quad
    \big[C_{xy}\big]_{ij} = \innerprod{\phi_i}{\mathcal{K} \phi_j}_{H_1}.
\end{equation*}
The matrix representation $ K \in \R^{n \times n} $ of the projected operator $ \mathcal{K}_\phi $ is then given by $ K = C_{xx}^{-1} C_{xy} $ so that
\begin{equation*}
    \mathcal{K}_\phi \ts f(x) = (K \ts \alpha)^\top \phi(x).
\end{equation*}
Assume now that $ K \ts \xi_\ell = \lambda_\ell \ts \xi_\ell $, then, defining $ \varphi_\ell(x) = \xi_\ell^\top \phi(x) $, we have
\begin{equation*}
    \mathcal{K}_\phi \ts \varphi_\ell(x) = (K \ts \xi_\ell)^\top \phi(x) = \lambda_\ell \ts \xi_\ell^\top \phi(x) = \lambda_\ell \ts \varphi_\ell(x).
\end{equation*}
That is, we can compute approximations of eigenvalues and eigenfunctions of the operator $ \mathcal{K} $ by computing eigenvalues and eigenvectors of the matrix $ K $.

\begin{lemma} \label{lem:Galerkin adjoint}
The matrix representation of the adjoint $ \mathcal{K}_\phi^* $ of $ \mathcal{K}_\phi $ is given by $ K^* = C_{yy}^{-1} C_{yx} $, where $ \big[C_{yy}\big]_{ij} = \innerprod{\phi_i}{\phi_j}_{H_2} $ and $ \big[C_{yx}\big]_{ij} = \big[C_{xy}\big]_{ji} $.
\end{lemma}

\begin{proof}
Given two functions $ f(x) = \alpha^\top \phi(x) $ and $ g(x) = \beta^\top \phi(x) $, we have $ \innerprod{f}{g}_{H_1} = \alpha^\top C_{xx} \ts \beta $ and $ \innerprod{f}{g}_{H_2} = \alpha^\top C_{yy} \ts \beta $. It follows that
\begin{align*}
    \innerprod{\mathcal{K}_\phi \ts f}{g}_{H_1} &= (K \ts \alpha)^\top C_{xx} \ts \beta = \alpha^\top C_{xy}^\top \ts \beta \\ &= \alpha^\top C_{yy} \ts (K^* \beta) = \innerprod{f}{\mathcal{K}_\phi^* \ts g}_{H_2}. \qedhere
\end{align*}
\end{proof}

Note that this is indeed the Galerkin approximation of the adjoint $ \mathcal{K}^* $ since $ \big[C_{yx}\big]_{ij} = \innerprod{\mathcal{K} \phi_i}{\phi_j}_{H_1} = \innerprod{\phi_i}{\mathcal{K}^* \phi_j}_{H_2} $.

%%%%%%%%%%%%%%%%%%%%%%%%%%%%%%%%%%%%%%%%%%%%%%%%%%%%%%%%%%%%%%%%%%%%%%%%%%%%%%%%%%%%%%%%%%
%%% BIBLIOGRAPHY %%%%
%%%%%%%%%%%%%%%%%%%%%%%%%%%%%%%%%%%%%%%%%%%%%%%%%%%%%%%%%%%%%%%%%%%%%%%%%%%%%%%%%%%%%%%%%%

\bibliographystyle{IEEEtran}
\bibliography{Graphons}

\begin{thebibliography}{10}
\providecommand{\url}[1]{#1}
\csname url@samestyle\endcsname
\providecommand{\newblock}{\relax}
\providecommand{\bibinfo}[2]{#2}
\providecommand{\BIBentrySTDinterwordspacing}{\spaceskip=0pt\relax}
\providecommand{\BIBentryALTinterwordstretchfactor}{4}
\providecommand{\BIBentryALTinterwordspacing}{\spaceskip=\fontdimen2\font plus
\BIBentryALTinterwordstretchfactor\fontdimen3\font minus
  \fontdimen4\font\relax}
\providecommand{\BIBforeignlanguage}[2]{{%
\expandafter\ifx\csname l@#1\endcsname\relax
\typeout{** WARNING: IEEEtran.bst: No hyphenation pattern has been}%
\typeout{** loaded for the language `#1'. Using the pattern for}%
\typeout{** the default language instead.}%
\else
\language=\csname l@#1\endcsname
\fi
#2}}
\providecommand{\BIBdecl}{\relax}
\BIBdecl

\bibitem{KT24}
S.~Klus and M.~Trower, ``Transfer operators on graphs: {S}pectral clustering
  and beyond,'' \emph{Journal of Physics: Complexity}, vol.~5, no.~1, p.
  015014, 2024.

\bibitem{KD24}
S.~Klus and N.~D. Conrad, ``Dynamical systems and complex networks: A {K}oopman
  operator perspective,'' \emph{Journal of Physics: Complexity}, vol.~5, no.~4,
  p. 041001, 2024.

\bibitem{LS06}
L.~Lovász and B.~Szegedy, ``Limits of dense graph sequences,'' \emph{Journal
  of Combinatorial Theory, Series B}, vol.~96, no.~6, pp. 933--957, 2006.

\bibitem{lovasz2012large}
L.~Lov{\'a}sz, \emph{Large networks and graph limits}.\hskip 1em plus 0.5em
  minus 0.4em\relax American Mathematical Society, 2012, vol.~60.

\bibitem{janson2013graphons}
S.~Janson, \emph{Graphons, cut norm and distance, couplings and
  rearrangements}, ser. New York Journal of Mathematics.\hskip 1em plus 0.5em
  minus 0.4em\relax State University of New York, University at Albany, Albany,
  NY, 2013, vol.~4.

\bibitem{PLC21}
J.~Petit, R.~Lambiotte, and T.~Carletti, ``Random walks on dense graphs and
  graphons,'' \emph{SIAM Journal on Applied Mathematics}, vol.~81, no.~6, pp.
  2323--2345, 2021.

\bibitem{APSS20}
M.~Avella-Medina, F.~Parise, M.~T. Schaub, and S.~Segarra, ``Centrality
  measures for graphons: Accounting for uncertainty in networks,'' \emph{IEEE
  Transactions on Network Science and Engineering}, vol.~7, no.~1, pp.
  520--537, 2020.

\bibitem{morency2021graphon}
M.~W. Morency and G.~Leus, ``Graphon filters: {G}raph signal processing in the
  limit,'' \emph{IEEE Transactions on Signal Processing}, vol.~69, pp.
  1740--1754, 2021.

\bibitem{RCR21}
L.~Ruiz, L.~F.~O. Chamon, and A.~Ribeiro, ``Graphon signal processing,''
  \emph{IEEE Transactions on Signal Processing}, vol.~69, pp. 4961--4976, 2021.

\bibitem{levie2023graphon}
R.~Levie, ``A graphon-signal analysis of graph neural networks,''
  \emph{Advances in Neural Information Processing Systems}, vol.~36, pp.
  64\,482--64\,525, 2023.

\bibitem{ruiz2020graphon}
L.~Ruiz, L.~Chamon, and A.~Ribeiro, ``Graphon neural networks and the
  transferability of graph neural networks,'' \emph{Advances in Neural
  Information Processing Systems}, vol.~33, pp. 1702--1712, 2020.

\bibitem{keriven2020convergence}
R.~Levie, ``Convergence and stability of graph convolutional networks on large
  random graphs,'' \emph{Advances in Neural Information Processing Systems},
  vol.~33, pp. 21\,512--21\,523, 2020.

\bibitem{neuman2023transferability}
A.~M. Neuman and J.~J. Bramburger, ``Transferability of graph neural networks
  using graphon and sampling theories,'' 2023.

\bibitem{borgs2008convergent}
C.~Borgs, J.~Chayes, L.~Lov{\'a}sz, V.~S{\'o}s, and K.~Vesztergombi,
  ``Convergent sequences of dense graphs {I}: {S}ubgraph frequencies, metric
  properties and testing,'' \emph{Advances in Mathematics}, vol. 219, no.~6,
  pp. 1801--1851, 2008.

\bibitem{parise2023graphon}
F.~Parise and A.~Ozdaglar, ``Graphon games: {A} statistical framework for
  network games and interventions,'' \emph{Econometrica}, vol.~91, no.~1, pp.
  191--225, 2023.

\bibitem{gao2019graphon}
S.~Gao and P.~E. Caines, ``Graphon control of large-scale networks of linear
  systems,'' \emph{IEEE Transactions on Automatic Control}, vol.~65, no.~10,
  pp. 4090--4105, 2019.

\bibitem{XMW23}
X.~Xia, G.~Mishne, and Y.~Wang, ``Implicit graphon neural representation,'' in
  \emph{International Conference on Artificial Intelligence and
  Statistics}.\hskip 1em plus 0.5em minus 0.4em\relax PMLR, 2023, pp.
  10\,619--10\,634.

\bibitem{LaMa94}
A.~Lasota and M.~C. Mackey, \emph{Chaos, fractals, and noise: Stochastic
  aspects of dynamics}, 2nd~ed., ser. Applied Mathematical Sciences.\hskip 1em
  plus 0.5em minus 0.4em\relax New York: Springer, 1994, vol.~97.

\bibitem{DJ99}
M.~Dellnitz and O.~Junge, ``On the approximation of complicated dynamical
  behavior,'' \emph{SIAM Journal on Numerical Analysis}, vol.~36, no.~2, pp.
  491--515, 1999.

\bibitem{Mezic05}
I.~Mezi{\'{c}}, ``Spectral properties of dynamical systems, model reduction and
  decompositions,'' \emph{Nonlinear Dynamics}, vol.~41, no.~1, pp. 309--325,
  2005.

\bibitem{Davies82a}
E.~B. Davies, ``Metastable states of symmetric {M}arkov semigroups {I},''
  \emph{Proceedings of the London Mathematical Society}, vol. s3-45, no.~1, pp.
  133--150, 1982.

\bibitem{SS13}
C.~Sch\"utte and M.~Sarich, \emph{Metastability and Markov State Models in
  Molecular Dynamics: Modeling, Analysis, Algorithmic Approaches}, ser. Courant
  Lecture Notes.\hskip 1em plus 0.5em minus 0.4em\relax American Mathematical
  Society, 2013, no.~24.

\bibitem{FrSaMo10}
G.~Froyland, N.~Santitissadeekorn, and A.~Monahan, ``Transport in
  time-dependent dynamical systems: Finite-time coherent sets,'' \emph{Chaos:
  An Interdisciplinary Journal of Nonlinear Science}, vol.~20, no.~4, p.
  043116, 2010.

\bibitem{Froyland13}
G.~Froyland, ``An analytic framework for identifying finite-time coherent sets
  in time-dependent dynamical systems,'' \emph{Physica D: Nonlinear Phenomena},
  vol. 250, pp. 1--19, 2013.

\bibitem{CA14}
S.~Chan and E.~Airoldi, ``A consistent histogram estimator for exchangeable
  graph models,'' in \emph{International Conference on Machine Learning}.\hskip
  1em plus 0.5em minus 0.4em\relax PMLR, 2014, pp. 208--216.

\bibitem{WKR15}
M.~O. Williams, I.~G. Kevrekidis, and C.~W. Rowley, ``A data-driven
  approximation of the {K}oopman operator: Extending dynamic mode
  decomposition,'' \emph{Journal of Nonlinear Science}, vol.~25, no.~6, pp.
  1307--1346, 2015.

\bibitem{KKS16}
S.~Klus, P.~Koltai, and C.~Sch{\"u}tte, ``On the numerical approximation of the
  {P}erron--{F}robenius and {K}oopman operator,'' \emph{Journal of
  Computational Dynamics}, vol.~3, no.~1, pp. 51--79, 2016.

\bibitem{Lov93}
L.~Lov\'asz, ``Random walks on graphs: {A} survey,'' \emph{Combinatorics, Paul
  Erdős is Eighty}, vol.~2, 1993.

\bibitem{BPS22}
B.~Bonnet, N.~P. Duteil, and M.~Sigalotti, ``Consensus formation in first-order
  graphon models with time-varying topologies,'' \emph{Mathematical Models and
  Methods in Applied Sciences}, vol.~32, no.~11, pp. 2121--2188, 2022.

\bibitem{Bishop06}
C.~M. Bishop, \emph{Pattern Recognition and Machine Learning}.\hskip 1em plus
  0.5em minus 0.4em\relax New York: Springer, 2006.

\bibitem{KWNS18}
P.~Koltai, H.~Wu, F.~No\'e, and C.~Sch\"utte, ``Optimal data-driven estimation
  of generalized {M}arkov state models for non-equilibrium dynamics,''
  \emph{Computation}, vol.~6, no.~1, 2018.

\bibitem{AR16}
S.~Athreya and A.~R{\"o}llin, ``Dense graph limits under respondent-driven
  sampling,'' \emph{The Annals of Applied Probability}, vol.~26, no.~4, pp.
  2193--2210, 2016.

\bibitem{NoNu13}
F.~No{\'e} and F.~N{\"u}ske, ``A variational approach to modeling slow
  processes in stochastic dynamical systems,'' \emph{Multiscale Modeling \&
  Simulation}, vol.~11, no.~2, pp. 635--655, 2013.

\bibitem{Ando62}
T.~Ando, ``On compactness of integral operators,'' \emph{Indagationes
  Mathematicae}, vol.~24, no.~2, pp. 235--239, 1962.

\bibitem{Graham79}
I.~G. Graham and I.~H. Sloan, ``On the compactness of certain integral
  operators,'' \emph{Journal of Mathematical Analysis and Applications},
  vol.~68, no.~2, pp. 580--594, 1979.

\bibitem{NC15}
F.~No{\'e} and C.~Clementi, ``Kinetic distance and kinetic maps from molecular
  dynamics simulation,'' \emph{Journal of Chemical Theory and Computation},
  vol.~11, no.~10, pp. 5002--5011, 2015.

\bibitem{bramburger2024persistence}
J.~Bramburger, M.~Holzer, and J.~Williams, ``Persistence of steady-states for
  dynamical systems on large networks,'' 2024.

\bibitem{KoMe18}
M.~Korda and I.~Mezi{\'{c}}, ``On convergence of extended dynamic mode
  decomposition to the {K}oopman operator,'' \emph{Journal of Nonlinear
  Science}, vol.~28, no.~2, pp. 687--710, 2018.

\bibitem{bramburger2024auxiliary}
J.~J. Bramburger and G.~Fantuzzi, ``Auxiliary functions as {K}oopman
  observables: {D}ata-driven analysis of dynamical systems via polynomial
  optimization,'' \emph{Journal of Nonlinear Science}, vol.~34, no.~1, p.~8,
  2024.

\bibitem{LDBK17}
Q.~Li, F.~Dietrich, E.~M. Bollt, and I.~G. Kevrekidis, ``Extended dynamic mode
  decomposition with dictionary learning: A data-driven adaptive spectral
  decomposition of the {K}oopman operator,'' \emph{Chaos: An Interdisciplinary
  Journal of Nonlinear Science}, vol.~27, no.~10, p. 103111, 2017.

\bibitem{MPWN18}
A.~Mardt, L.~Pasquali, H.~Wu, and F.~No{\'e}, ``{VAMPnets} for deep learning of
  molecular kinetics,'' \emph{Nature Communications}, vol.~9, 2018.

\bibitem{TLK25}
M.~Tabish, B.~Leimkuhler, and S.~Klus, ``How deep is your network? {D}eep vs.\
  shallow learning of transfer operators,'' 2025.

\bibitem{ZZ23}
C.~Zhang and E.~Zuazua, ``A quantitative analysis of {K}oopman operator methods
  for system identification and predictions,'' \emph{Comptes Rendus.
  M\'ecanique}, 2023.

\bibitem{LLLK24}
L.~Llamazares-Elias, S.~Llamazares-Elias, J.~Latz, and S.~Klus, ``Data-driven
  approximation of {K}oopman operators and generators: Convergence rates and
  error bounds,'' 2024.

\bibitem{BaKo17}
R.~Banisch and P.~Koltai, ``Understanding the geometry of transport: Diffusion
  maps for {L}agrangian trajectory data unravel coherent sets,'' \emph{Chaos:
  An Interdisciplinary Journal of Nonlinear Science}, vol.~27, no.~3, p.
  035804, 2017.

\bibitem{KHMN19}
S.~Klus, B.~E. Husic, M.~Mollenhauer, and F.~No\'e, ``Kernel methods for
  detecting coherent structures in dynamical data,'' \emph{Chaos: An
  Interdisciplinary Journal of Nonlinear Science}, vol.~29, no.~12, 2019.

\bibitem{LNP13}
T.~Leli{\`e}vre, F.~Nier, and G.~A. Pavliotis, ``Optimal non-reversible linear
  drift for the convergence to equilibrium of a diffusion,'' \emph{Journal of
  Statistical Physics}, vol. 152, no.~2, pp. 237--274, 2013.

\bibitem{BKKBDS18}
A.~Bittracher, P.~Koltai, S.~Klus, R.~Banisch, M.~Dellnitz, and C.~Sch{\"u}tte,
  ``Transition manifolds of complex metastable systems: {T}heory and
  data-driven computation of effective dynamics,'' \emph{Journal of Nonlinear
  Science}, vol.~28, no.~2, pp. 471--512, 2018.

\bibitem{Luxburg07}
U.~von Luxburg, ``A tutorial on spectral clustering,'' \emph{Statistics and
  Computing}, vol.~17, no.~4, pp. 395--416, 2007.

\bibitem{TDK25}
M.~Trower, N.~D. Conrad, and S.~Klus, ``Clustering time-evolving networks using
  the spatiotemporal graph {L}aplacian,'' \emph{Chaos}, vol.~35, no.~1, p.
  013126, 2025.

\bibitem{Aubin00}
J.-P. Aubin, \emph{Applied functional analysis}.\hskip 1em plus 0.5em minus
  0.4em\relax New York: John Wiley \& Sons, 2000.

\bibitem{MSKS20}
M.~Mollenhauer, I.~Schuster, S.~Klus, and C.~Sch\"utte, ``Singular value
  decomposition of operators on reproducing kernel {H}ilbert spaces,'' in
  \emph{Advances in Dynamics, Optimization and Computation}.\hskip 1em plus
  0.5em minus 0.4em\relax Cham: Springer, 2020, pp. 109--131.

\end{thebibliography}

%\newpage



% \end{IEEEbiography}



\end{document}